\documentclass[lettersize,journal]{IEEEtran}
\usepackage{amsmath,amsfonts}
\usepackage{algorithmic}
\usepackage{algorithm}
\usepackage{array}
\usepackage{amsthm}
\usepackage[caption=false,font=normalsize,labelfont=sf,textfont=sf]{subfig}
\usepackage{textcomp}
\usepackage{stfloats}
\usepackage{url}
\usepackage{verbatim}
\usepackage{graphicx}
\usepackage{cite}
\usepackage{siunitx}
\newtheorem{theorem}{$\mathbf{Theorem}$}

\newtheorem{definition}{$\mathbf{Definition}$}

\DeclareMathOperator*{\argmin}{arg\,min}

\begin{document}

\title{Guarding Force: Safety-Critical Compliant Control for Robot-Environment Interaction}

\author{
	\vskip 1em
	Xinming Wang, \emph{Graduate Student Member},
	Jun Yang, \emph{IEEE Fellow},
	Jianliang Mao, \emph{IEEE Member},
	\\Jinzhuo Liang, Shihua Li, \emph{IEEE Fellow}, and Yunda Yan \emph{IEEE Member}
	
	\thanks{Xinming Wang and Shihua Li are with the Key Laboratory of Measurement and Control of CSE, Ministry of Education, School of Automation, Southeast University, Nanjing, 210096, China. (email: wxm\_seu@seu.edu.cn, lsh@seu.edu.cn).}
	\thanks{Jun Yang is with the Department of Aeronautical and Automotive Engineering, Loughborough University, Loughborough LE11 3TU, UK (e-mail: j.yang3@lboro.ac.uk).}
	\thanks{Jianliang Mao and Jinzhuo Liang are with the College of Automation Engineering, Shanghai University of Electric Power, Shanghai, 201306, China (e-mail: jl\_mao@shiep.edu.cn, jz\_liang@mail.shiep.edu.cn).}
	\thanks{Yunda Yan is with the Department of Computer Science, University College London, London, WC1E 6BT, UK (e-mail: yunda.yan@ucl.ac.uk).}
}



\maketitle

\begin{abstract}
In this study, we propose a safety-critical compliant control strategy designed to strictly enforce interaction force constraints during the physical interaction of robots with unknown environments. The interaction force constraint is interpreted as a new force-constrained control barrier function (FC-CBF) by exploiting the generalized contact model and the prior information of the environment, i.e., the prior stiffness and rest position, for robot kinematics. The difference between the real environment and the generalized contact model is approximated by constructing a tracking differentiator, and its estimation error is quantified based on Lyapunov theory. By interpreting strict interaction safety specification as a dynamic constraint, restricting the desired joint angular rates in kinematics, the proposed approach modifies nominal compliant controllers using quadratic programming, ensuring adherence to interaction force constraints in unknown environments. The strict force constraint and the stability of the closed-loop system are rigorously analyzed. Experimental tests using a UR3e industrial robot with different environments verify the effectiveness of the proposed method in achieving the force constraints in unknown environments.
\end{abstract}

\def\abstractname{Note to Practitioners}
\begin{abstract}
Robotic manipulators are increasingly involved in complex interaction missions with humans and various environments. As these interactions become more prevalent, ensuring the safety of both robots and their human and environmental counterparts has become imperative. Although advanced controllers with safety constraints have been explored, ensuring strict interaction force guarantees remains challenging, particularly in unknown environments. To address this challenge, this paper proposes a new safety-critical compliant control scheme using the control barrier function (CBF) technique. By formulating the approach as quadratic programming and dynamically adjusting the nominal compliant controller online, the proposed approach acts as a safety mechanism, ensuring strict interaction force guarantees and preventing potential damage. 
\end{abstract}


\section{Introduction}
In transitioning to Industry 5.0, robotic manipulators are assuming increasingly pivotal roles \cite{demir2019industry,cotta2021towards}. No longer confined to repetitive tasks, they are now expected to navigate alongside human collaborators in complex interactions within unfamiliar environments. Examples range from robotic-assisted surgeries \cite{su2019improved} to autonomous exploration \cite{calisi2007multi} and industrial operations \cite{jinno1995development}. However, as robotic manipulators enter these diverse interaction scenarios, safety emerges as a paramount concern. Ensuring the well-being of both the robot and the environment it interacts with requires compliant behavior during these interactions. Compliance enables the robot to dynamically adapt its actions, minimizing the risk of harm while enhancing efficiency and effectiveness. 

A traditional method to improve the controllability of the manipulator during interactions is to incorporate damping elements into the robot to absorb unexpected impacts when the robot interacts with the unknown environment \cite{van2009compliant,wolf2015variable}. However, the integration of additional mechanical components can pose challenges in maintenance and contribute to an increase in the overall weight and size of the robot. To avoid these inconveniences, an alternative approach is the impedance/admittance control \cite{hogan1985impedance,ott2010unified}. This method adjusts the robot's compliance by considering the desired impedance relationship between the environment and the robot end-effector. Through precise adjustment of impedance parameters, it ensures that compliance and safety requirements are met during interactions in an energy efficient way \cite{abu2020variable}, particularly when the mechanical characteristics of the environment, such as stiffness, are known.  

Nevertheless, in the majority of the aforementioned interaction scenarios, the mechanical properties of the environment, e.g., the stiffness and rest position, are typically unknown or inaccurate. In this situation, conventional impedance/interference control methods may fail to ensure essential safety conditions, such as maintaining interaction forces within a secure range. An illustrative example is the occurrence of force shocks when the robot engages in an environment that has a stiffness different from what was expected \cite{bicchi2004fast,haddadin2007safety}. To achieve the desired impedance in an unknown environment, advanced impedance approaches are investigated, especially utilizing adaptive/learning methods \cite{luo1993control,li2013impedance,liang2023adaptive} to improve control performance with contact-rich tasks by actively learning from the environment. However, they do not guarantee the explicitly strict interaction force constraint during the learning process. The other straightforward approach to preventing excessive contact force is to immediately stop the robot when the force exceeds a critical threshold. However, this method may lead to inevitable oscillations. 

To regulate the contact force during the robot's interaction with an unknown environment, this paper introduces a \textbf{S}afety-\textbf{C}ritical \textbf{C}ompliant \textbf{C}ontrol (SC$^3$) approach, using the control barrier function (CBF) technique \cite{Ames2017, ferraguti2022safety, peng2023robust}. To achieve the interaction force constraint in an unknown environment, a generalized form of the interaction contact model is initially used to approximate real-world conditions. Leveraging this model and incorporating prior knowledge about the environment, such as stiffness and rest position, the environment uncertainty and mismatch are estimated by constructing a tracking differentiator and quantifying estimation errors via Lyapunov theory. Then, taking into account the impacts caused by the uncertainty of the contact model and the estimation error, we formulate a force-constrained CBF (FC-CBF) for the kinematics of the robot, restricting the desired joint angular velocity to achieve the strict force constraint. By modifying well-designed compliant controllers like the impedance/admittance control subject to the dynamic constraints, the strict interaction force constraints are guaranteed during robot interactions in unknown environments. Rigorous analysis of interaction force constraints and the stability of the manipulator is also presented. 

Our method is systematically evaluated in real-world experiments to validate its effectiveness. In the experimental test, the interaction task is conducted on the UR3e industrial robot with different physical characteristics, including elasticity and viscoelasticity. Compared with conventional admittance control, the results demonstrate the effectiveness of the proposed method in achieving strict force constraints in unknown environments. Furthermore, to verify the adaptability of the proposed approach, it is also applied to force control tasks with a hybrid characteristic environment.

Additionally, we also discuss the feasibility of extending our approach to the dynamic system of the robotic manipulator, where a new FC-CBF is designed to impose constraints on the driving torques of each joint. The contributions of this work are summarized as follows:
\begin{enumerate}
	\item We propose the safety-critical compliant control strategy for the robot-environment interaction by leveraging the CBF technique. The proposed approach acts as a flexible safety patch for most compliant control methodologies, avoiding the complexity of controller redesign. 
	\item Using prior knowledge of the stiffness matrix and rest position, we develop an effective uncertainty estimator aimed at dynamically approximating the impacts arising from the complex environment uncertainties. Furthermore, the estimation error is rigorously quantified using the Lyapunov theory under mild assumptions. 
	\item Through extensive experimentation, employing various mechanical materials such as springs and sponges, conducted on a UR3e robot, we validate the efficacy of our proposed approach. In addition, we perform force control tasks to demonstrate the adaptability and versatility of our approach.
\end{enumerate}

\section{Related work}
To meet the safety requirements necessary for the interaction between the robot and the environment, researchers have explored various approaches in existing literature. In this section, we offer a concise overview of some relevant studies in this domain. 
\subsection{Passive compliant control}
A direct method to improve safety and mitigate unexpected impacts during interaction involves specific design strategies for robotics \cite{van2009compliant,wolf2015variable}. For example, an elastomeric shear pad-based compliance device has been incorporated into the end effectors to facilitate safe peg-in-hole tasks \cite{southern2002study}. Passive compliance can also be achieved by employing passive mechanisms to generate suitable reactions to applied forces in each joint. A typical example is the series elastic actuator \cite{pratt1995series,zhong2021toward}, where a spring is connected in series with a stiff actuator. However, while introducing a spring improves compliance, its stiffness is typically fixed and determined by the spring, limiting adaptability during operation and restricting practical application. To address this limitation, the concept of variable impedance or stiffness actuators has been proposed and applied in various robotic applications, including children's toys \cite{stiehl2005design,saldien2006anty}, walking/running robots \cite{huang2012step,liu2019impedance}, and robot hands for throwing tasks \cite{wolf2008new,grebenstein2011dlr}. Nevertheless, the increased mechanical complexity associated with these designs may pose challenges in terms of maintenance, overall weight increase, dynamics modeling, and controller design.

\subsection{Active compliant control}
Unlike passive compliance control that relies on adding damping elements or changing actuators, achieving safe and compliant behavior can also be achieved by developing active control algorithms to adjust the compliance characteristics of the robot \cite{schumacher2019introductory}. One such approach is the hybrid force/position control method, where the entire task is divided into the position-controlled task and the force-controlled task \cite{reibert1981hybrid,xie2023bi}. This approach allows for regulating the interaction force along a specific direction when there is prior knowledge of the structure and geometry of the environment. However, this approach has limitations in unstructured or dynamic environments. Another typical active compliance control method is impedance/admittance control \cite{hogan1985impedance,ott2010unified}, where the impedance relationship is approximated as a mass-spring-damping system with specialized impedance parameters. Once the environment is known, the desired compliance and safety requirements can be met by adjusting these parameters. In the presence of unknown environments, adaptive/learning algorithms have been developed. For example, when the dynamic environment has parameter uncertainties, an adaptive algorithm is proposed for robot manipulators to autonomously improve the control performance \cite{luo1993control}. In \cite{li2013impedance}, an impedance learning method is proposed to iteratively adjust the impedance parameters of the robot arm so that the desired impedance model is obtained in changing unknown environments. Although adaptability to unfamiliar environments is enhanced, there is still no explicitly strict interaction force constraint during the learning and manipulation processes.

\subsection{CBF-based safety-critical control for robot}
Control barrier functions (CBF)-based control has emerged as an effective tool for formally guaranteeing the safety of nonlinear systems through the establishment of a forward invariant set \cite{Ames2017, ferraguti2022safety, peng2023robust}. Due to its ability to describe complex nonlinear constraints and its relatively low computational requirements, CBF-based control has seen significant advances in both theoretical research \cite{jankovic2018robust,Xiao2021} and its applications in robotics, including obstacle avoidance \cite{nguyen2021robust}, robotic grasping \cite{cortez2019control}, wire-borne brachiating robots \cite{farzan2022adaptive}, and human-robot collaboration \cite{shi2023dual}. In particular, in the realm of bipedal robot locomotion, robust CBFs are utilized to constrain foot-step displacements under different carrying loads for the RABBIT bipedal robot \cite{nguyen2021robust}. In the field of robot grasping \cite{cortez2019control}, CBFs are used to interpret constraints related to contact forces, joint angles, and contact locations. In the presence of strong uncertainty and unmodeled dynamics, robust CBFs are deployed to restrict the position of a swing gripper for safe brachiation maneuvers by a cable-suspended underactuated robot \cite{farzan2022adaptive}. According to the ISO/TS 15066 technical specification,a new CBF is developed to limit robot velocity and ensure human safety in human-robot collaboration scenarios \cite{shi2023dual}. In contrast to the distance/velocity-based safety constraints addressed using various CBF designs, our previous work \cite{liang2023control} focused on achieving strict force constraints. This was accomplished through the development of a new CBF, using a known-environment model that includes the stiffness matrix and the rest position of the environment. Building upon this, this paper extends the investigation to explore interaction force constraints in unknown environments.

\section{Mathematical notation}
The sets of real numbers and nonnegative integers are denoted as $\mathbb{R}$ and $\mathbb{N}$. For $i,j\in\mathbb{N}$ satisfying $j\leq k$, define $\mathbb{N}_{j:k}=\{ j,j+1,\cdots,k \}$ as a subset of $\mathbb{N}$. The identity matrix is denoted as $\boldsymbol{I}_n$ with size $n$. For any vectors $\boldsymbol{x}, \boldsymbol{y} \in\mathbb{R}^n$, $\boldsymbol{x} \leq \boldsymbol{y}$ corresponds to the element-wise inequality between vectors $\boldsymbol{x}$ and $\boldsymbol{y}$, i.e., $x_i\leq y_i, i\in\mathbb{N}_{1:n}$. A continuous function $\alpha_e:(-b, a) \rightarrow (-\infty,\infty)$ is said to belong to extended class $\mathcal{K}$ function for some $a, b >0$ if it is strictly increasing and $\alpha_e(0)=0$. Define $\Vert \boldsymbol{x} \Vert$ as the 2-norm of vector $\boldsymbol{x}$.
\section{Preliminaries and problem formulation}
In this section, the mathematical model of the manipulator, a brief introduction of the control barrier function, and the modeling of the environment interaction will be presented subsequently to facilitate the design of the proposed safety-critical compliant control strategy.
\subsection{Mathematical model of the manipulator}
The dynamic model of the serial robot can be presented as follows.
\begin{equation}\label{robot_dyqsys}
	\boldsymbol{M}(\boldsymbol{q})\boldsymbol{\ddot{q}} + \boldsymbol{C}(\boldsymbol{q}, \dot{\boldsymbol{q}})\dot{\boldsymbol{q}} + \boldsymbol{G}(\boldsymbol{q}) = \boldsymbol{\tau} - \boldsymbol{\tau}_e,
\end{equation}
where $\boldsymbol{q}, \dot{\boldsymbol{q}}\in\mathbb{R}^n$ are the vectors of joint angles' positions and velocities and $n$ is the degree-of-freedom, $\boldsymbol{\tau},\boldsymbol{\tau}_e \in\mathbb{R}^n$ are the operation torque and external torque generated by interaction with environment/human, respectively. The positive definite matrix $\boldsymbol{M}(\boldsymbol{q})$ is the inertial matrix, $\boldsymbol{C}(\boldsymbol{q},\dot{\boldsymbol{q}})\dot{\boldsymbol{q}}$ represents the effect caused by the centrifugal and Coriolis forces, and $\boldsymbol{G}(\boldsymbol{q})$ denotes the gravity torque on the robot. According to the principle of virtual work, the external torque $\boldsymbol{\tau}_e$ can be rewritten as
\begin{equation}
	\boldsymbol{\tau}_e = \boldsymbol{J}^T(\boldsymbol{q})\boldsymbol{h}_e,
\end{equation}
where $\boldsymbol{h}_e:=[\boldsymbol{f}_e^T, \boldsymbol{\mu}_e^T]^T$ represents the interaction force and moment exerted on the end-effector, and $\boldsymbol{f}_e$ and $\boldsymbol{\mu}_e$ denote the vectors of external end-effector force and moment, respectively. 

During the interaction, the physical interaction is usually described in Cartesian space. Define $\boldsymbol{x}\in\mathbb{R}^m,\;m>0$ as the pose of the end effector in a Cartesian space. It can be presented by using the forward kinematics transformation $\boldsymbol{x} = \boldsymbol{K}(\boldsymbol{q})$, where $\boldsymbol{K}(\boldsymbol{q})$ is the transform matrix from the base to the end effector. It should be noted that the matrix $\boldsymbol{K}(\boldsymbol{q})$ is dependent on the geometry structure of the robot. The velocity of the end effector in Cartesian space is presented as
\begin{equation}\label{ef_vel}
	\dot{\boldsymbol{x}} = \boldsymbol{J}(\boldsymbol{q})\dot{\boldsymbol{q}},
\end{equation}
where $\boldsymbol{J}({\boldsymbol{q}})$ is the Jacobian matrix. For most industrial robots, the dynamic loop is not open for the user to directly design the driving torque for each joint. In this situation, $\dot{\boldsymbol{q}}$ can be regarded as the control input to be designed. 
\subsection{CBF-based safety control}
In the following, the base of control barrier function-based safety control is introduced. Consider the following nonlinear system
\begin{equation}
	\dot{\boldsymbol{x}} = \boldsymbol{f}(\boldsymbol{x}) + \boldsymbol{g}(\boldsymbol{x})\boldsymbol{u},
	\label{cbf_sys}
\end{equation}
where $\boldsymbol{x}\in \mathbb{X} \subseteq \mathbb{R}^n$, $\boldsymbol{u}\in \mathbb{U} \subseteq \mathbb{R}^m$ are the states and the control inputs, $\mathbb{X}$ and $\mathbb{U}$ are the admissible sets of state and input. The nonlinear functions $\boldsymbol{f}:\mathbb{R}^n\rightarrow\mathbb{R}^n$, $\boldsymbol{g}:\mathbb{R}^n\rightarrow\mathbb{R}^{n\times m}$ are Lipschitz continuous functions. Then, for any initial condition $\boldsymbol{x}(t_0)\in\mathbb{R}^n$, there exists a maximum time interval $I(\boldsymbol{x}(t_0))= [t_0,\tau_{max})$ such that $\boldsymbol{x}(t)$ is the unique solution to system (\ref{cbf_sys}) on $I(\boldsymbol{x}(t_0))$. We assume that there exists a well-designed controller $\boldsymbol{u}_{0}:\mathbb{X}\rightarrow \mathbb{U}$ achieving the stability control tasks.

To describe the desired safety specification, consider a continuously differentiable function $b:\mathbb{R}^n\rightarrow\mathbb{R}$ and define the following 0-superlevel set 
\begin{equation}
	\begin{aligned}
		C = \{ \boldsymbol{x}\in\mathbb{R}^n \vert b(\boldsymbol{x}) \geq 0 \}, 
	\end{aligned}
\end{equation}
where $\partial C:= \{ \boldsymbol{x}\in\mathbb{R}^n \vert b(\boldsymbol{x})=0 \}$ and ${\rm Int}(C):= \{ \boldsymbol{x}\in\mathbb{R}^n \vert b(\boldsymbol{x})>0 \}$ are the boundary and interior of the set $C$. In the framework of CBF, the system (\ref{cbf_sys}) is safe with respect to the set $C$, if there exists a control input $\boldsymbol{u}$ makes the set $C$ a forward invariant set, i.e., the trajectory $\boldsymbol{x}(t)\in C, \; \forall t>0$, if $\boldsymbol{x}(0) \in C$.
\begin{definition}{\cite{Ames2017}}
	Considering the nonlinear system (\ref{cbf_sys}), the continuously differentiable function $b(\boldsymbol{x})$ is a CBF, if there exists an extended class $\mathcal{K}$ function $\alpha_e(\cdot)$ subject to
	\begin{equation}
		\sup_{u\in\mathbb{U}} \{ L_{\boldsymbol{f}}b(\boldsymbol{x}) + L_{\boldsymbol{g}}b(\boldsymbol{x})\boldsymbol{u} + \alpha_e(b(\boldsymbol{x})) \}\geq 0, \; \forall \boldsymbol{x}\in\mathbb{X},
	\end{equation} 
	where $L_{\boldsymbol{f}}b(\boldsymbol{x}) := \frac{\partial b(\boldsymbol{x})}{\partial \boldsymbol{x}}\boldsymbol{f}(\boldsymbol{x})$ and $L_{\boldsymbol{g}}b(\boldsymbol{x}) := \frac{\partial b(\boldsymbol{x})}{\partial \boldsymbol{x}}\boldsymbol{g}(\boldsymbol{x})$ are standard Lie derivatives. 
\end{definition}
If there exists a valid CBF $b(\boldsymbol{x})$, then the CBF-based safety controller can be obtained by solving the following quadratic programming.
\begin{equation}\label{cbf_based}
	\begin{aligned}
		\argmin_{u\in \mathbb{\mathbb{U}}}& \; \frac{1}{2}\Vert\boldsymbol{u} - \boldsymbol{u}_0\Vert^2,\\
		&{\rm{s.t.}}\; L_{\boldsymbol{f}}b(\boldsymbol{x}) + L_{\boldsymbol{g}}b(\boldsymbol{x})\boldsymbol{u} + \alpha_e(b(\boldsymbol{x})) \geq 0.
	\end{aligned}
\end{equation}
In this sense, the safety task described by the invariance of the set $C$ can be guaranteed by modifying the nominal controller $\boldsymbol{u}_0$ in a minimal way. The typical trajectory of system (\ref{cbf_sys}) under controller (\ref{cbf_based}) is presented in Fig. \ref{fig1}(a).
\begin{figure}[t]
	\centering
	\vspace{0pt}
	\includegraphics[width=0.4\textwidth]{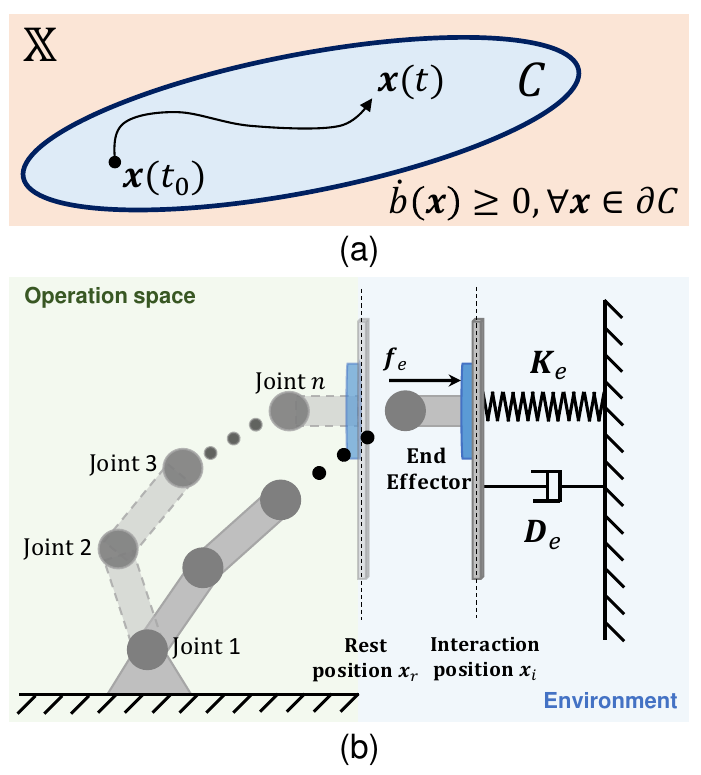}
	\caption{Illustrations of the notion of CBF and a typical physical interaction task: (a) Under the control action generated by (\ref{cbf_based}), the trajectory $\boldsymbol{x}(t)\in C,\;\forall t\geq 0$. Moreover, any $\boldsymbol{x}\in\partial C$, the controller (\ref{cbf_based}) will make $\boldsymbol{x}(t)$ move forward the ${\rm{Int}} C$; (b) The illustrative robot-environment interaction task with spring and damper environment.}
	\label{fig1}
\end{figure}
\subsection{Environment model and problem formulation}
In the most of compliant control approaches for robot-environment physical interaction, the interaction environment model is usually regarded as a pure `spring-like' environment \cite{siciliano1999robot}, and the interaction force and moment $\boldsymbol{h}_e$ can be presented as 
\begin{equation}
	\boldsymbol{h}_e = \left\{
	\begin{array}{cc}
		\boldsymbol{K}_e\Delta\boldsymbol{x},\; & \textbf{in contact} \\
		\boldsymbol{0},\; & \textbf{separated}
	\end{array}
	\right.,
\end{equation}
where $\boldsymbol{K}_e$ and $\Delta \boldsymbol{x}:=\boldsymbol{x}_i - \boldsymbol{x}_r$ are the environment's stiffness matrix and its deformation, $\boldsymbol{x}_i$ is the interaction surface position and $\boldsymbol{x}_r$ is the rest position. However, this model cannot accurately capture the physical characteristics of some materials with viscoelasticity such as sponge, wood, and rubber \cite{de2012viscoelasticity,morro2017modelling}, which are common in real life. In this regard, we consider the following generalized environment in this paper
\begin{equation}\label{env}
	\boldsymbol{h}_e = \left\{
	\begin{array}{cc}
		\boldsymbol{K}_e\Delta\boldsymbol{x} + \boldsymbol{\sigma}_e(t, \Delta\boldsymbol{x}, \Delta\dot{\boldsymbol{x}}),\; & \textbf{in contact} \\
		\boldsymbol{0},\; & \textbf{separated}
	\end{array}
	\right.,
\end{equation}
where $\boldsymbol{\sigma}_e(t, \Delta\boldsymbol{x}, \Delta\dot{\boldsymbol{x}})$ is the nonlinear function describing the viscosity or unmodelled interaction characteristics. In an unknown environment, the real information of $\boldsymbol{K}_e$, $\boldsymbol{x}_r$ and $\boldsymbol{\delta}_e$ is difficult to obtain exactly or estimate before the interaction \cite{erickson2003contact}, dramatically increasing the difficulty of preventing contact force violation. An illustrative robot-environment interaction task with the spring and damper environment is shown in Fig. \ref{fig1}(b), where $\boldsymbol{K}_e$ and $\boldsymbol{D}_e$ are the unknown stiffness and damping matrices of the environment. 

The purpose of this study is to design a safety-critical control approach for the manipulator system (\ref{ef_vel}) to regulate the motion of a robot to guarantee the strict interaction force constraint in the unknown environment described by the interaction model (\ref{env}). 
\section{Safety-critical compliant control strategy}
In the following, we assume that there exists a known prior stiffness matrix $\boldsymbol{K}_e^{pri}$ representing the expected stiffness characteristics of the contacted surface and an prior rest position of the environment $\boldsymbol{x}_r^{pri}$. The proposed compliant control strategy is first proposed for the kinematics of the manipulator (\ref{ef_vel}). With knowledge of $\boldsymbol{K}_e^{pri}$ and $\boldsymbol{x}_r^{pri}$, the generalized interaction force can be represented as 
\begin{equation}\label{env2}
	\boldsymbol{h}_e = \left\{
	\begin{array}{cc}
		\boldsymbol{K}_e^{pri}(\boldsymbol{x} - \boldsymbol{x}_r^{pri}) + \boldsymbol{\delta}_e(t, \Delta\boldsymbol{x}, \Delta\dot{\boldsymbol{x}}),\; & \textbf{in contact} \\
		\boldsymbol{0},\; & \textbf{separated}
	\end{array}
	\right.,
\end{equation}
with
\begin{equation}
	\begin{aligned}
		\boldsymbol{\delta}_e(t, \Delta\boldsymbol{x}, \Delta\dot{\boldsymbol{x}}) =& (\boldsymbol{K}_e - \boldsymbol{K}_e^{pri})\boldsymbol{x}-\boldsymbol{K}_e\boldsymbol{x}_r+\boldsymbol{K}_e^{pri}\boldsymbol{x}_r^{pri}\\
		& + \boldsymbol{\sigma}_e(t, \Delta\boldsymbol{x}, \Delta\dot{\boldsymbol{x}}),
	\end{aligned}
\end{equation}
where $\boldsymbol{\delta}_e(t, \Delta\boldsymbol{x}, \Delta\dot{\boldsymbol{x}})$ is the lumped interaction uncertainty incorporating inaccurately mechanical mismatch as well as the unmodelled interaction characteristics. It should be noted that $\boldsymbol{\delta}_e$ is measurable if the robot has the force and moment sensor, that is, $\boldsymbol{\delta}_e = \boldsymbol{h}_e^{mea} - \boldsymbol{K}_e^{pri}(\boldsymbol{x} - \boldsymbol{x}_r^{pri})$, where $\boldsymbol{h}_e^{mea}$ is the measured external force/moment obtained by the sensor.
\subsection{Tracking differentiator design}
Before introducing the proposed interaction force constrained CBF, the following tracking differentiator is constructed to estimate the time derivative of $\boldsymbol{\delta}_e(t, \Delta\boldsymbol{x}, \Delta\dot{\boldsymbol{x}})$, which will be used in the force-constrained CBF design. 
\begin{equation}\label{td_sys}
	\begin{aligned}
		\dot{\boldsymbol{z}}_1 &= \boldsymbol{z}_2 + \boldsymbol{L}_1(\boldsymbol{\delta}_e - \boldsymbol{z}_1),\\
		\dot{\boldsymbol{z}}_2 &= \boldsymbol{L}_2(\boldsymbol{\delta}_e - \boldsymbol{z}_1),
	\end{aligned}
\end{equation}
where $\boldsymbol{z}_1$ and $\boldsymbol{z}_2$ are the estimates of $\boldsymbol{\delta}_e$ and $\dot{\boldsymbol{\delta}}_e$, $\boldsymbol{L}_1$ and $\boldsymbol{L}_2$ are the positive definite diagonal matrices to be designed. Defining the estimation error as $\tilde{\boldsymbol{Z}} = [\tilde{\boldsymbol{z}}_1^T, \tilde{\boldsymbol{z}}_2^T]^T$ and $ \tilde{\boldsymbol{z}}_1 = \boldsymbol{\delta}_e - \boldsymbol{z}_1,\; \tilde{\boldsymbol{z}}_2=\dot{\boldsymbol{\delta}}_e - \boldsymbol{z}_2$, it is obtained
\begin{equation}\label{td_errsys}
	\begin{aligned}
		\dot{\tilde{\boldsymbol{Z}}} = \boldsymbol{A}\tilde{\boldsymbol{Z}} + \boldsymbol{\delta},
	\end{aligned}
\end{equation}
with 
\begin{displaymath}
	\begin{aligned}
		\boldsymbol{A} &= \begin{bmatrix}
			-\boldsymbol{L}_1 & \boldsymbol{I}_6 \\
			-\boldsymbol{L}_2 & \boldsymbol{0}
		\end{bmatrix}, \boldsymbol{\delta} = \begin{bmatrix}
			\boldsymbol{0}_{6\times 1} \\
			\ddot{\boldsymbol{\delta}}_e
		\end{bmatrix}.
	\end{aligned}
\end{displaymath}
Assuming that there exists a known positive constant $\beta\in\mathbb{R}^+$ satisfying $\Vert \boldsymbol{\delta}(t)\Vert\leq \beta,\; \forall t>0$, then we can construct a monotonically decreasing function to quantify the estimation error $\tilde{\boldsymbol{Z}}$. Considering the Lyapunov function $V=\frac{1}{2}\tilde{\boldsymbol{Z}}^T\tilde{\boldsymbol{Z}}$, it yields
\begin{equation}\label{td_dV}
	\begin{aligned}
		\dot{V} &= \tilde{\boldsymbol{Z}}^T\boldsymbol{A}\tilde{\boldsymbol{Z}} + \tilde{\boldsymbol{Z}}^T\boldsymbol{\delta}\\
		&\leq -2(\alpha - \frac{\epsilon}{2})V + \frac{\beta^2}{2\epsilon},
	\end{aligned}
\end{equation}  
where $\alpha$ is a positive parameter satisfying $\tilde{\boldsymbol{Z}}^T\boldsymbol{A}\tilde{\boldsymbol{Z}}\leq -\alpha\Vert\tilde{\boldsymbol{Z}}\Vert^2$ and $\epsilon>0$. From (\ref{td_dV}), the following estimation error bound can be constructed.
\begin{equation}\label{td_errbound}
	\Vert \tilde{\boldsymbol{Z}}(t)\Vert \leq \sqrt{\Vert\boldsymbol{Z}(0)\Vert^2 e^{-at} + (1 - e^{-at})b}:=\bar{z}(t),
\end{equation} 
where $a = 2(\alpha - \frac{\epsilon}{2})$ and $b = \frac{\beta^2}{2\alpha\epsilon - \epsilon^2}$.
\begin{figure}[!t]
	\centering
	\vspace{0pt}
	\includegraphics[width=0.45\textwidth]{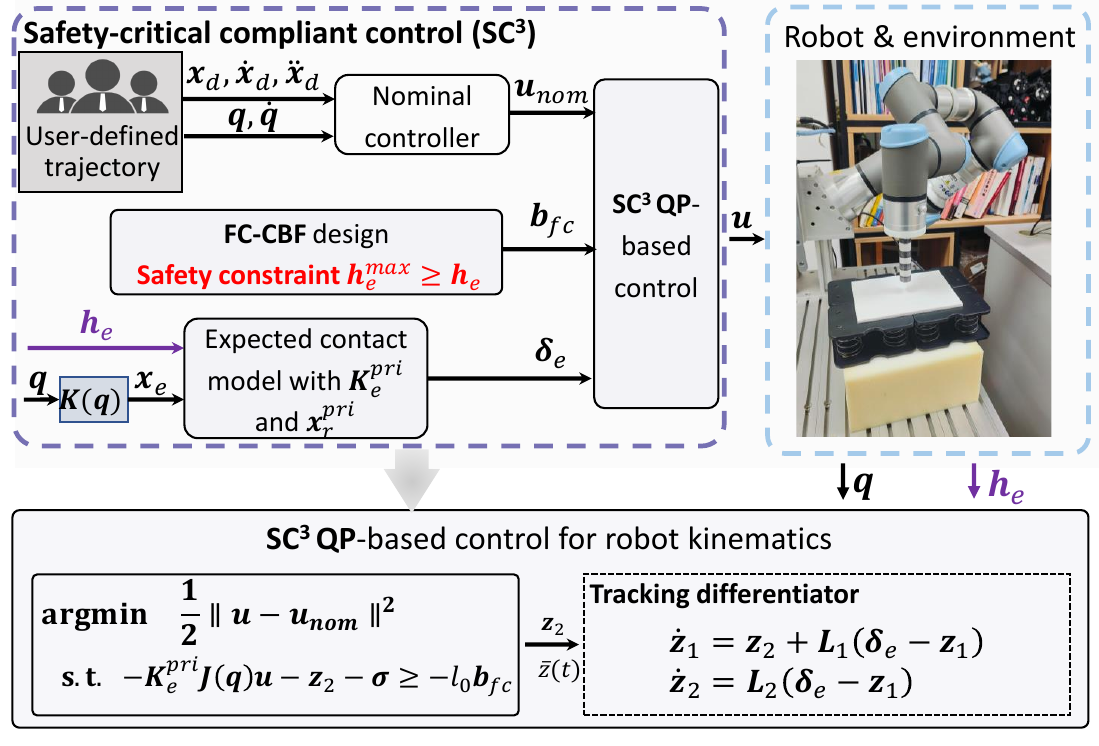}
	\caption{Schematic block of the proposed safety-critical compliant control strategy.}
	\label{fig2}
\end{figure}
\subsection{Force-constrained control barrier function design}
The strict interaction force constraint is described as the following piece-wise continuous function
\begin{equation}\label{FC_CBF}
	\boldsymbol{b}_{fc}(t,\Delta \boldsymbol{x}) = \left\{
	\begin{array}{cc}
		\boldsymbol{h}_e^{max} - \boldsymbol{h}_e,\; & \textbf{in contact} \\
		\boldsymbol{0},\; & \textbf{separated}
	\end{array}
	\right.,
\end{equation}
where $\boldsymbol{h}_e^{max}\in\mathbb{R}^6$ is the maximum limit of the interaction forces and moments. Inspired by the CBF, if there exists a control $\dot{\boldsymbol{q}}(t)$ such that the function $\boldsymbol{b}_{fc}$ is always positive, then the interaction force constraint can be guaranteed. The set of relative invariances with respect to $\boldsymbol{x}$ is defined as $\boldsymbol{C}_0=\{ [\boldsymbol{x}^T,\dot{\boldsymbol{x}}^T]^T\in\mathbb{R}^{12}\vert \boldsymbol{b}_{fc}(t,\Delta\boldsymbol{x})\geq 0 \}$. The proposed FC-CBF design is now presented as follows.
\begin{definition}
	(FC-CBF for manipulator kinematics) Consider the kinematic system (\ref{ef_vel}) and the tracking differentiator (\ref{td_sys}). The function $\boldsymbol{b}_{fc}$ is the FC-CBF for the system (\ref{ef_vel}) during interaction, if there exists a positive parameter $l$ such that the following condition holds.
	\begin{equation}
		\sup_{\dot{\boldsymbol{q}}\in\mathbb{R}^m}
			\{-\boldsymbol{K}_e^{pri}\boldsymbol{J}(\boldsymbol{q})\dot{\boldsymbol{q}} - \boldsymbol{z}_2 - \boldsymbol{\sigma} + l\boldsymbol{b}_{fc}(t,\Delta \boldsymbol{x}) \}\geq 0,			
		\label{FCCBF1}
	\end{equation}
	where $\boldsymbol{\sigma} = [\bar{z}(t),\bar{z}(t)\ldots,\bar{z}(t)]^T$.
\end{definition}
Following the virtues of CBF, we can prove that the interaction force constraint can be achieved with a valid FC-CBF.
\begin{theorem}
	Consider the kinematic system (\ref{ef_vel}) and the tracking differentiator (\ref{td_sys}). Using $\boldsymbol{b}_{fc}$ defined in (\ref{FC_CBF}), any Lipschitz continuous controller belonging to $\mathcal{K}_k = \{ \dot{\boldsymbol{q}}\in\mathbb{R}^m\vert -\boldsymbol{K}_e^{pri}\boldsymbol{J}(\boldsymbol{q})\dot{\boldsymbol{q}} - \boldsymbol{z}_2 - \boldsymbol{\sigma}(t) + l\boldsymbol{b}_{fc}(t,\Delta \boldsymbol{x}) \geq 0 \}$ guarantees the strict force constraint during the physical interaction.
\end{theorem}
\begin{proof}
	During the physical interaction, the time derivative of $\boldsymbol{b}_{fc}$ along the kinematics robot system is
	\begin{equation}\label{db_fc1}
		\begin{aligned}
			\dot{\boldsymbol{b}}_{fc}(t, \Delta \boldsymbol{x}) = -\boldsymbol{K}_e^{pri}\boldsymbol{J}(\boldsymbol{q})\dot{\boldsymbol{q}} - \dot{\boldsymbol{\delta}}_e(t).
		\end{aligned}
	\end{equation} 
	Substituting any Lipschitz continuous controller from the set $\mathcal{K}_k$ into the (\ref{db_fc1}), it is obtained that
	\begin{equation}\label{db_fc2}
		\begin{aligned}
			\dot{\boldsymbol{b}}_{fc}(t, \Delta \boldsymbol{x}) &\geq \boldsymbol{\sigma} + \boldsymbol{z}_2 - \dot{\boldsymbol{\delta}}_e - lb_{fc}(t,\Delta \boldsymbol{x})\\
			&\geq - l\boldsymbol{b}_{fc}(t,\Delta \boldsymbol{x}).
		\end{aligned}
	\end{equation}
	The second estimate is derived following the quantification of the estimation error obtained in (\ref{td_errbound}), that is, $\boldsymbol{\sigma}(t)\geq\boldsymbol{z}_2-\dot{\boldsymbol{\delta}}_e$. Since $\boldsymbol{\delta}_e$ is measurable with the sensor, then (\ref{db_fc2}) can be rewritten as follows
	\begin{equation}
		\dot{\boldsymbol{b}}_{fc}(t, \Delta \boldsymbol{x}) \geq - l(\boldsymbol{h}_e^{max} - \boldsymbol{h}_e).
	\end{equation} 
	
	Note that $\boldsymbol{b}_{fc}(t, \Delta \boldsymbol{x}) = \boldsymbol{h}_e^{max}-\boldsymbol{h}_e$. It is clear that any controller from the set $\mathcal{K}_k$ renders $\boldsymbol{h}_e^{max}-\boldsymbol{h}_e \geq - l(\boldsymbol{h}_e^{max} - \boldsymbol{h}_e)$ such that the force constraint is guaranteed during the physical interaction. 
\end{proof}
\subsection{Safety-critical compliant control design}
To integrate the interaction force constraint with the nominal robot position tracking task, the proposed control is derived by solving the following SC$^3$ quadratic programming (SC$^3$ QP). 
\begin{equation}
	\begin{aligned}
		\dot{\boldsymbol{q}}^*=&\argmin_{\dot{\boldsymbol{q}}\in \mathbb{R}^n} \; \frac{1}{2}\Vert\dot{\boldsymbol{q}} - \dot{\boldsymbol{q}}_{nom}\Vert^2, \quad\quad\quad\quad\quad\quad\quad\quad \textbf{SC$^3$ QP}\\
		&\; {\rm{s.t.}}\; -\boldsymbol{K}_e^{pri}\boldsymbol{J}(\boldsymbol{q})\dot{\boldsymbol{q}} - \boldsymbol{z}_2 - \boldsymbol{\sigma}(t) + l\boldsymbol{b}_{fc}(t,\Delta \boldsymbol{x}) \geq 0,
	\end{aligned}
	\label{QR_FCCBF}
\end{equation}
where the nominal position tracking controller $\dot{\boldsymbol{q}}_{nom}$ can be any well-designed position tracking controllers. The SC$^3$ QP modifies a baseline tracking controller $\dot{\boldsymbol{q}}_{nom}(t)$, subject to the FC-CBF constraint, in a minimally invasive manner. With the control approach proposed in (\ref{QR_FCCBF}), a strict interaction force constraint can be achieved and avoid the complexity of the force-constrained controller redesign. 

\begin{figure*}[!t]
	\centering
	\vspace{0pt}
	\includegraphics[width=0.9\textwidth]{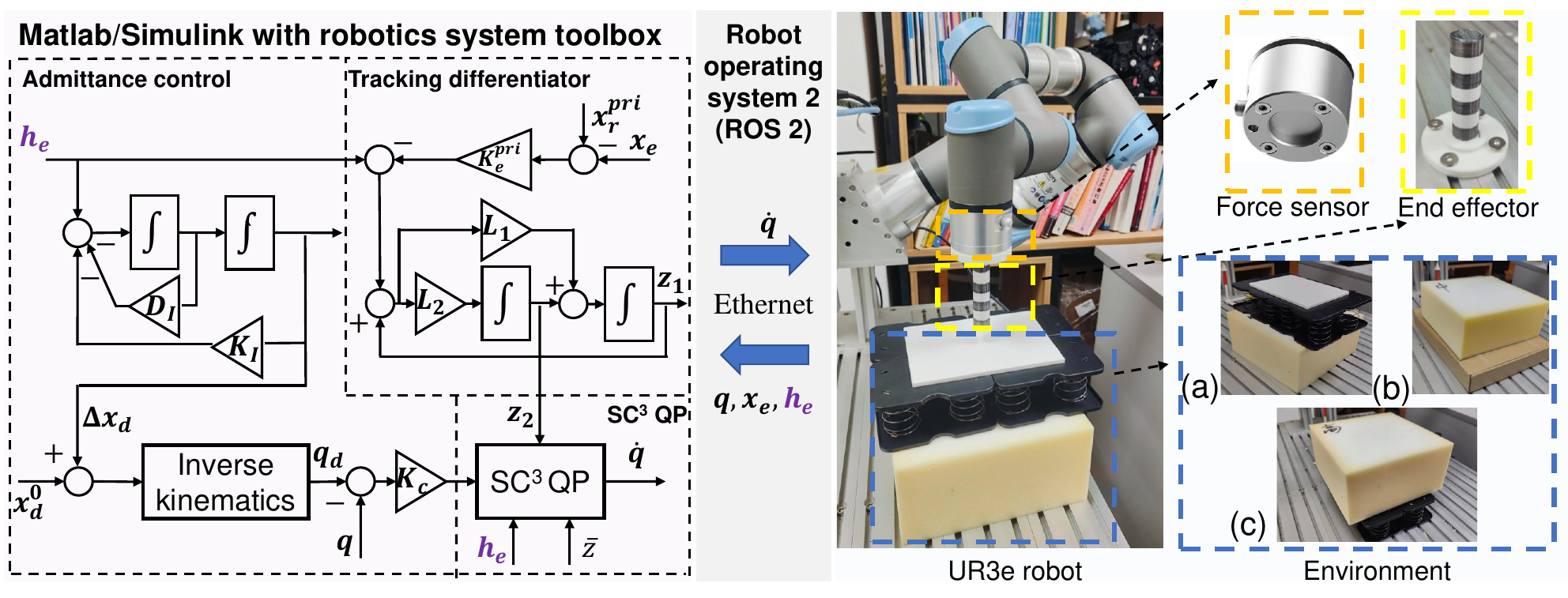}
	\caption{The setup of the experimental test on a UR3e robot: (a) the elastic environment (spring); (b) the viscoelastic environment (sponge); (c) the hybrid characteristic environment (a combination of sponge and spring in series).}
	\label{fig4}
\end{figure*} 
Under the SC$^3$ QP, the stability of the manipulator is discussed in the following theorem.
\begin{theorem}
	Consider the kinematic system (\ref{ef_vel}) and the tracking differentiator (\ref{td_sys}). Under SC$^3$ QP (\ref{QR_FCCBF}), all signals of the closed-loop system of the manipulator are bounded. Additionally, if the FC-CBF constraint is inactive, the nominal task can be achieved under $\dot{\boldsymbol{q}}_{nom}$. Upon activation of the FC-CBF constraint, the system transitions to force control, ensuring convergence of $\boldsymbol{h}_e$ to $\boldsymbol{h}_e^{max}$.
\end{theorem}
\begin{proof}
	The Lagrangian function of SC$^3$ QP is presented as follows.
	\begin{equation}
		\begin{aligned}
			L = \frac{1}{2}\Vert \dot{\boldsymbol{q}} - \dot{\boldsymbol{q}}_{nom} \Vert^2 + \sum_{i=1}^{6}\lambda_i(-\boldsymbol{g}_i\dot{\boldsymbol{q}}+z_{2,i}+\sigma_i-lb_{fc,i}),
		\end{aligned}
	\end{equation} 
	where $\lambda_i,\;i\in\mathbb{N}_{1:6}$ are Lagrangian multiples, $\boldsymbol{g}_i$ is the $i$th row vector of matrix $-\boldsymbol{K}_e^{pri}\boldsymbol{J}(\boldsymbol{q})$, $z_{2,i}$, $\sigma_i$ and $b_{fc,i}$ are the $i$th elements of the $\boldsymbol{z}_2$, $\boldsymbol{\sigma}$ and $\boldsymbol{b}_{fc}$. 
	
	Using the Karush-Kuhn-Tucker (KKT) condition in \cite{boyd2004convex}, we first solve the explicit form of (\ref{QR_FCCBF}), which is presented as follows
	\begin{equation}
		\dot{\boldsymbol{q}}^*= \left\{
		\begin{array}{c c}
			\dot{\boldsymbol{q}}_{nom}, & \boldsymbol{q}\in \Omega_0 \\
			\dot{\boldsymbol{q}}_{nom} + \Delta\dot{\boldsymbol{q}}, & \boldsymbol{q}\in \Omega_\mathcal{I} 		
		\end{array}\right.,
		\label{u_SC3}
	\end{equation}
	where $\Delta\dot{\boldsymbol{q}}$ is the modification term induced from the interaction force constraints. The $\Omega_0$ and $\Omega_\mathcal{I}$ are the sets in which the constraints are inactive or not, and their explicit form depends on the number of active constraints. It should be noted that $\Omega_0 \cup \Omega_\mathcal{I} = \mathbb{R}^6$. Next, the stability of the closed-loop system is analyzed depending on whether $\boldsymbol{q}$ is in $\Omega_0$ or ${\Omega}_\mathcal{I}$. 
	
	\textbf{Case 1}: The nominal controller $\dot{\boldsymbol{q}}_{nom}$ satisfies all the interaction force constraints. Since the $\dot{\boldsymbol{q}}_{nom}$ is assumed to be effective achieving the basic control tasks, there exists a positive definite function $V(\boldsymbol{q})$ such that $\dot{V}(\boldsymbol{q})$ is negative positive definite under the $\dot{\boldsymbol{q}}_{nom}$. Then, the nominal tracking task can be achieved, and all the signals of the closed-loop system are bounded.
	
	\textbf{Case 2}: At least one interaction constraint is active. We first consider a simple case, i.e., the $i$th constraint is active and the others are inactive. In this case, the Lagrangian multiple is  $\lambda_i=\frac{-\boldsymbol{g}_i\dot{\boldsymbol{q}}_{nom}+z_{2,i}+\sigma_i-lb_{fc,i}}{\Vert \boldsymbol{g}_i\boldsymbol{g}_i^T \Vert}$, and the set $\Omega_\mathcal{I} := \{\boldsymbol{q}\in\mathbb{R}^6| \boldsymbol{g}_i\dot{\boldsymbol{q}}_{nom}-z_{2,i}-\sigma_i+lb_{fc,i} < 0 \}$.
	
	Consider the Lyapunov function $V_i = \frac{1}{2}\xi_i^2,\; \xi_i = h_{e,i}^{max} - h_{e,i}$, where $h_{e,i}^{max}$, $h_{e,i}$ are the $i$th element of $\boldsymbol{h}_e^{max}$ and $\boldsymbol{h}_e$. Since the FC-CBF constraint (\ref{FCCBF1}) is satisfied, the time derivative of $V_i$ is
	\begin{equation}
		\begin{aligned}
			\dot{V}_i &= \xi_i(-lb_{fc,i} + z_{2,i} + \sigma_i) - \xi_i\dot{\delta}_{e,i}.
		\end{aligned}
	\end{equation}
	Note that $b_{fc,i} = h_{e,i}^{max} - h_{e,i}$ and there exists a positive constant $\epsilon_1>0$ such that $\vert z_{2,i} + \sigma_i - \dot{\delta}_{e,i}\vert < \epsilon_1$ from (\ref{td_errbound}). Then, it is obtained that
	\begin{equation}
		\dot{V}_i \leq -\alpha V_i + \beta,
	\end{equation}
	where $\alpha = 2(l-\frac{1}{2})$ and $\beta = \frac{1}{2}\epsilon_1^2$. It is clear that, if one of the interaction constraints is active, the manipulator will transition to force control mode, and the force tracking error exponentially converges to a compact set of zero. According to (\ref{ef_vel}) and (\ref{env}), it is obtained that the position of the end effector and the joint angles are bounded. 
	
	The above analysis is applicable to complex scenarios where multiple constraints are active simultaneously. 
\end{proof}

In the following, experimental tests of the UR3e robot are performed to illustrate the effectiveness of the proposed control strategy.
\section{Experimental test on UR3e robot}
The interaction tests are conducted using the industrial robot UR3e, which features mature low-level controllers to track desired joint velocities. Control algorithms are developed on the Matlab/Simulink platform with the Robotics System Toolbox, running on a Lenovo laptop equipped with an i7-13700H CPU and 32 GB RAM. Real-time data including joint positions, end effector pose, and measured force are collected by Robot Operating System 2 (ROS 2) and transmitted to Matlab / Simulink software at a frequency of 50 Hz. The experimental setup comprises the UR3e robot, the materials that comprise the tested environment, and the SC$^3$ approach based on admittance control, as illustrated in Fig. \ref{fig4}. To show the effectiveness of the proposed approach, the admittance control is considered for the comparison in which the impedance parameters are set as $\boldsymbol{K}_I = 600\boldsymbol{I}_6$ and $\boldsymbol{D}_I = 40\boldsymbol{I}_6$. The feedback gain used in the joint space is set as $\boldsymbol{K}_c=0.2\boldsymbol{I}_6$, and the parameters of the tracking differentiator are $\boldsymbol{L}_1 = 110\boldsymbol{I}_6$ and $\boldsymbol{L}_2 = 3000\boldsymbol{I}_6$.

\subsection{Repetitive interaction with different environments}
In this subsection, the interaction tasks under unknown environments with different physical characteristics such as elasticity (spring), viscoelasticity (sponge), and hybrid characteristics (a combination of spring and sponge in series) are considered. In the following test, the robot will follow a vertical trajectory, frequently pressing and then leaving the unknown environment. A square position reference is used, where the initial position along the $Z$ axis is denoted by $z_{\text{initial}}=0.045m$ and the desired position is denoted by $z_{\text{desired}}=-0.005m$ with respect to the robot base. 

\begin{figure}[!t]
	\centering
	\vspace{0pt}
	\includegraphics[width=0.45\textwidth]{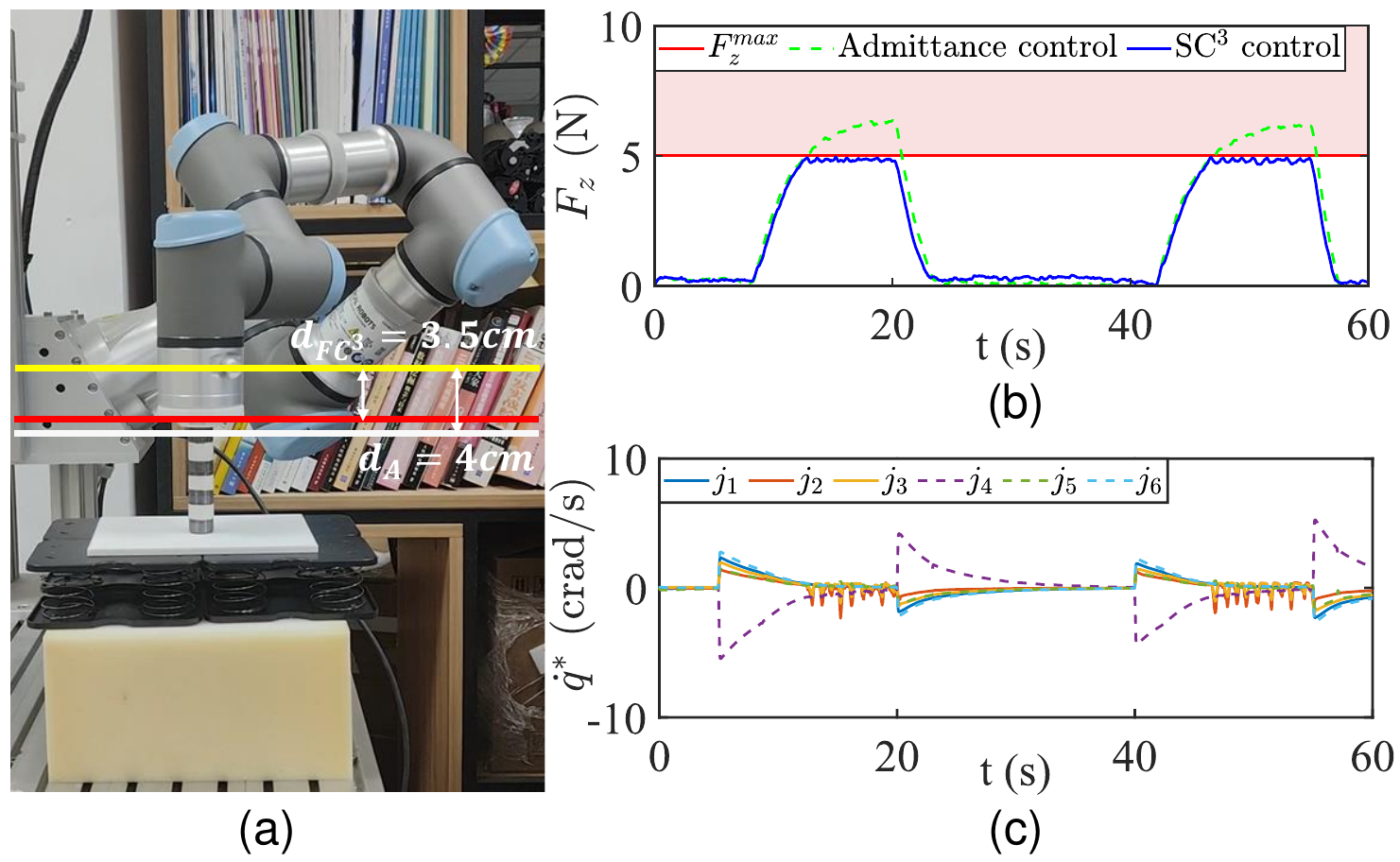}
	\caption{Experiment results for interaction task with the elastic environment: (a) The snap of UR3e pressing the spring under the SC$^3$ control. The yellow line indicates the initial position of the end-effector base, the red and white lines are the deepest position reached with SC$^3$ control and admittance control, respectively; (b) Curves of interaction forces; (c) Curves of desired joint velocities under SC$^3$.}
	\label{fig4}
\end{figure}

The actual rest position of the spring surface is indicated by $z_{\text{rest}}=0.011m$, while the prior rest position is denoted by $z_{\text{rest}}^{\text{pri}}=0m$. The prior stiffness is chosen as $\boldsymbol{K}_{e,z}^{\text{pri}} = 200  \text{N/m}$. The control parameter of FC-CBF is set as $l=10$.

\textit{Interaction test I (elastic environment):} This interaction task involves an unknown stiffness spring. The force constraint along $Z$-axis in this case is $\boldsymbol{F}_z^{\text{max}} = 5 \text{N}$. 

The interaction results, depicted in Fig. \ref{fig4}, include the following components: a snapshot of the UR3e pressing the spring with SC$^3$ control, curves illustrating interaction forces, and desired joint velocities. From Fig. \ref{fig4} (b), it is evident that the admittance control fails to adhere to the force constraint, whereas the proposed SC$^3$ approach effectively regulates the interaction force, even when the real rest position and stiffness of the environment are not precisely known. This observation is further illustrated in Fig. \ref{fig4} (a), where the deepest position achieved with SC$^3$ control is smaller than that obtained with admittance control. In Fig. \ref{fig4} (c), the desired joint velocities with SC$^3$ control during the interaction phases are presented.

\textit{Interaction test II (viscoelastic environment)}: In this interaction task, a high-density sponge is considered, commonly used in equipment protection. The force constraint in this case is $\boldsymbol{F}_z^{\text{max}} = 5 \text{N}$.
\begin{figure}[!t]
	\centering
	\vspace{0pt}
	\includegraphics[width=0.43\textwidth]{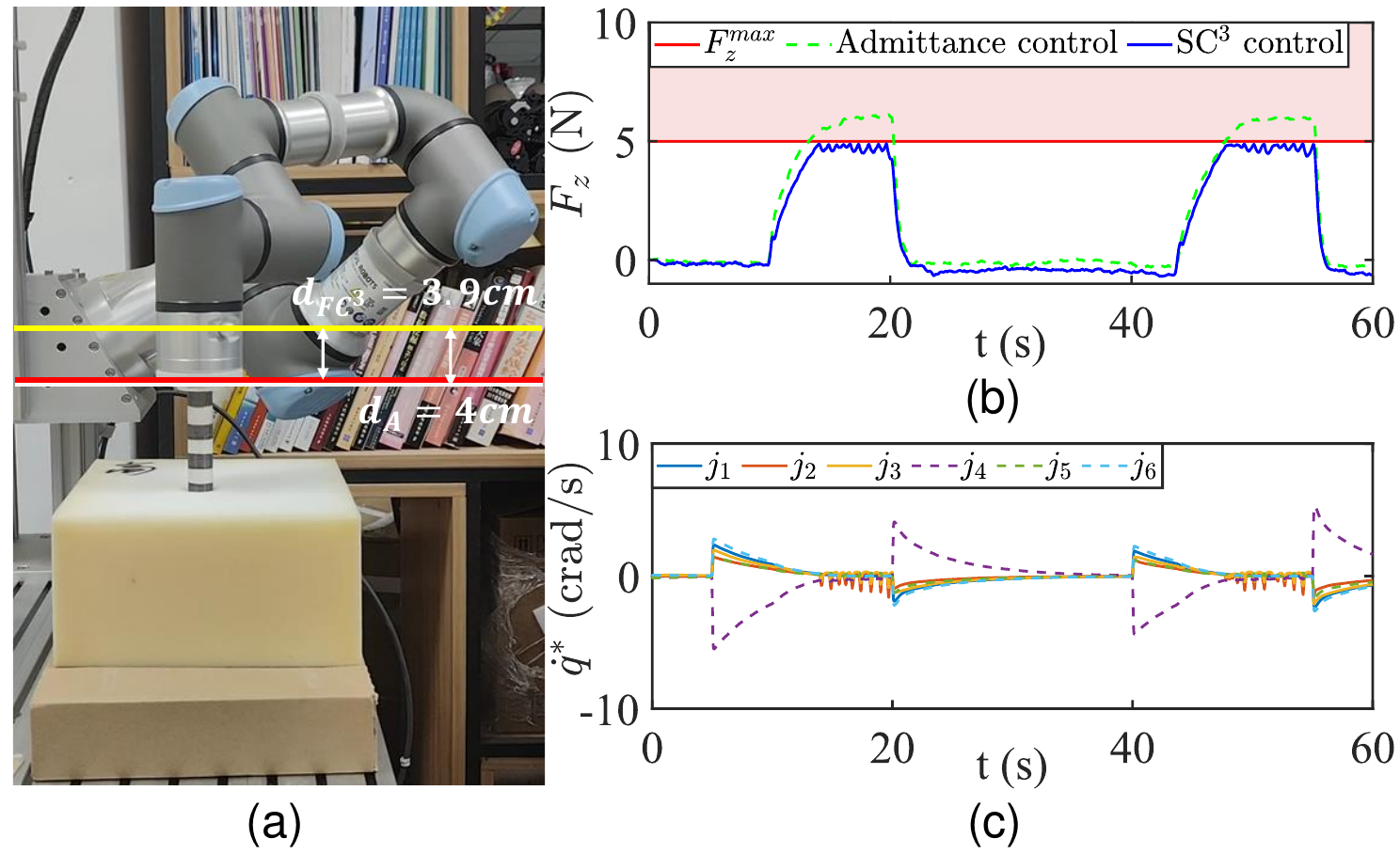}
	\caption{Experiment results for interaction task with the viscoelastic environment: (a) The snap of UR3e pressing the sponge under the SC$^3$ control. The yellow line indicates the initial position of the end-effector base, the red and white lines are the deepest position reached with SC$^3$ control and admittance control, respectively; (b) Curves of interaction forces; (c) Curves of desired joint velocities under SC$^3$.}
	\label{fig5}
\end{figure}

The interaction results shown in Fig. \ref{fig5} contain: a snapshot of UR3e pressing the high-density sponge with SC$^3$ control, curves illustrating interaction forces and desired joint velocities. From Fig. \ref{fig5} (b), it is evident that the proposed SC$^3$ approach can achieve strict interaction force constraints with the unknown rest position and stiffness of the sponge environment. The position difference between SC$^3$ and the admittance control, shown in Fig. \ref{fig5} (a), further verifies the effectiveness of the proposed approach.

\textit{Interaction test III (hybrid characteristic environment):} This interaction task involves a more complex environment by integrating a spring and a high-density sponge in series. This configuration allows the spring to absorb some force during the pressing phase and release it when the robot is moving away, thereby introducing additional complexity to the physical interaction. The force constraint considered is $\boldsymbol{F}_z^{\text{max}} = 3 \text{N}$.
\begin{figure}[!h]
	\centering
	\vspace{0pt}
	\includegraphics[width=0.45\textwidth]{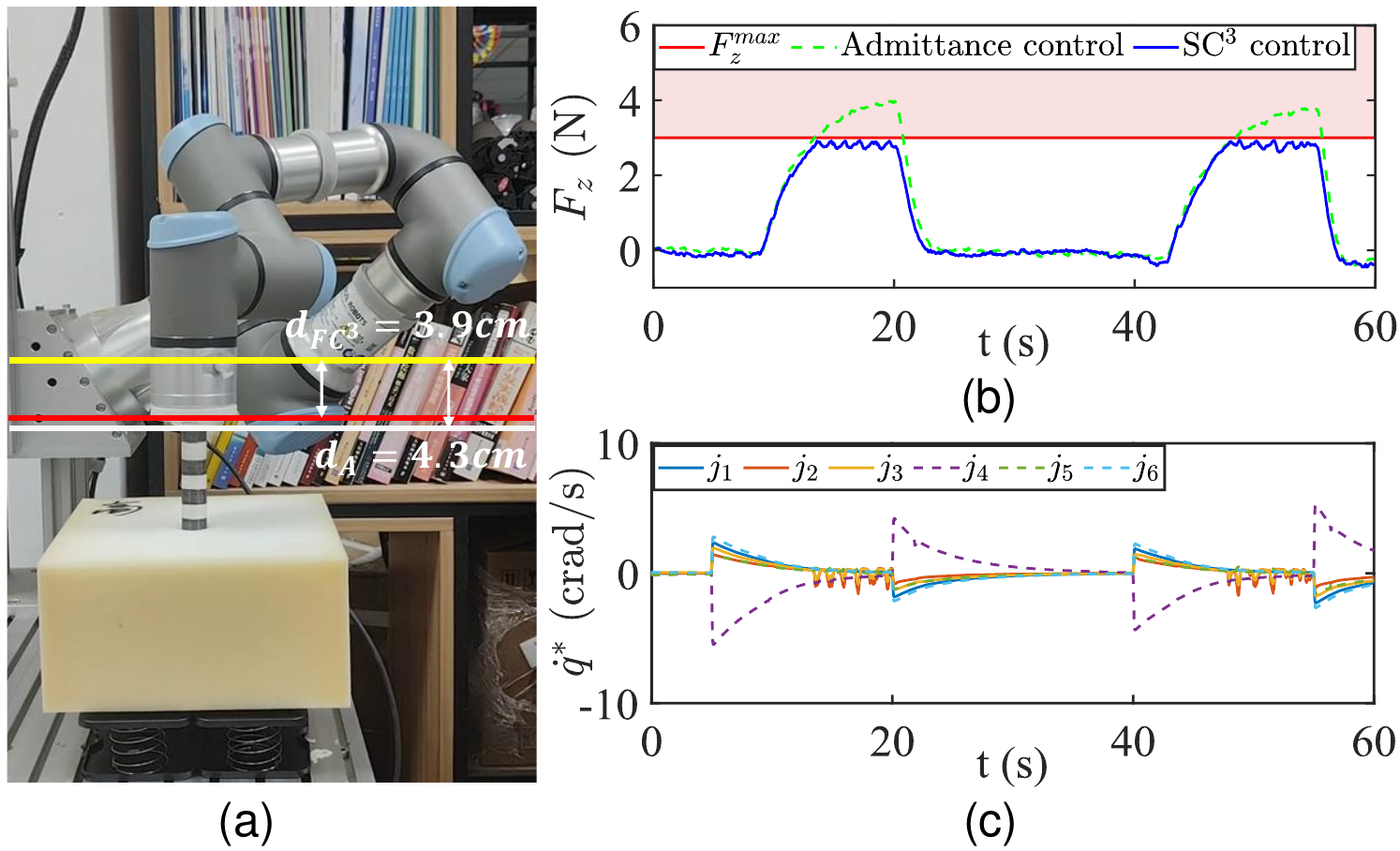}
	\caption{Experiment results for interaction task with the hybrid characteristic environment: (a) The snap of UR3e pressing the complex environment under SC$^3$ control. The yellow line indicates the initial position of the base of the end effector, the red and white lines are the deepest position reached with SC$^3$ control and admittance control, respectively; (b) Curves of interaction forces; (c) Curves of the desired joint velocities under SC$^3$.}\label{fig6}
\end{figure}
\begin{figure*}[!t]
	\centering
	\subfloat[ ]{
		\includegraphics[width=0.8\textwidth]{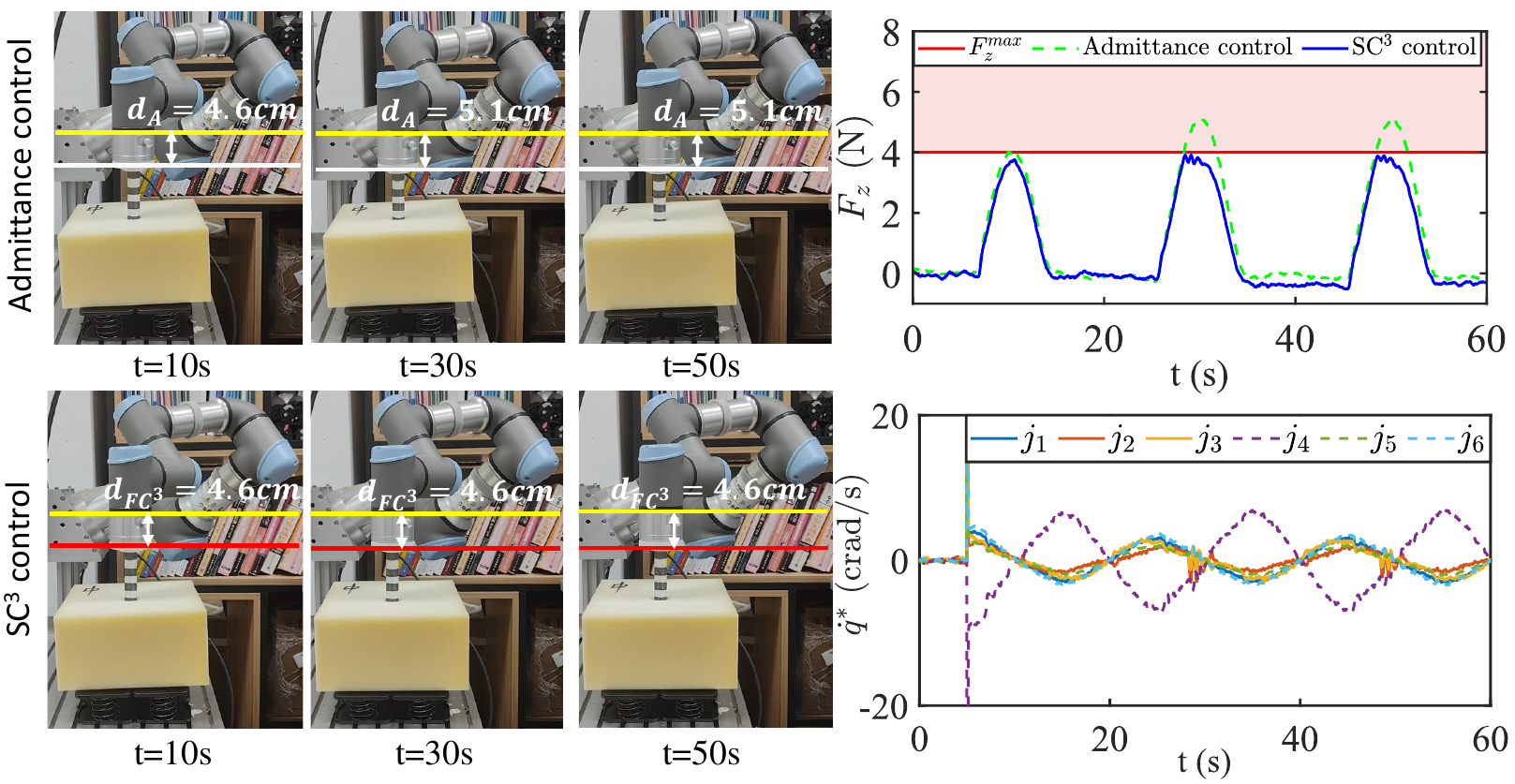}
		\label{fig7 1}}\\
	\subfloat[ ]{
		\includegraphics[width=0.8\textwidth]{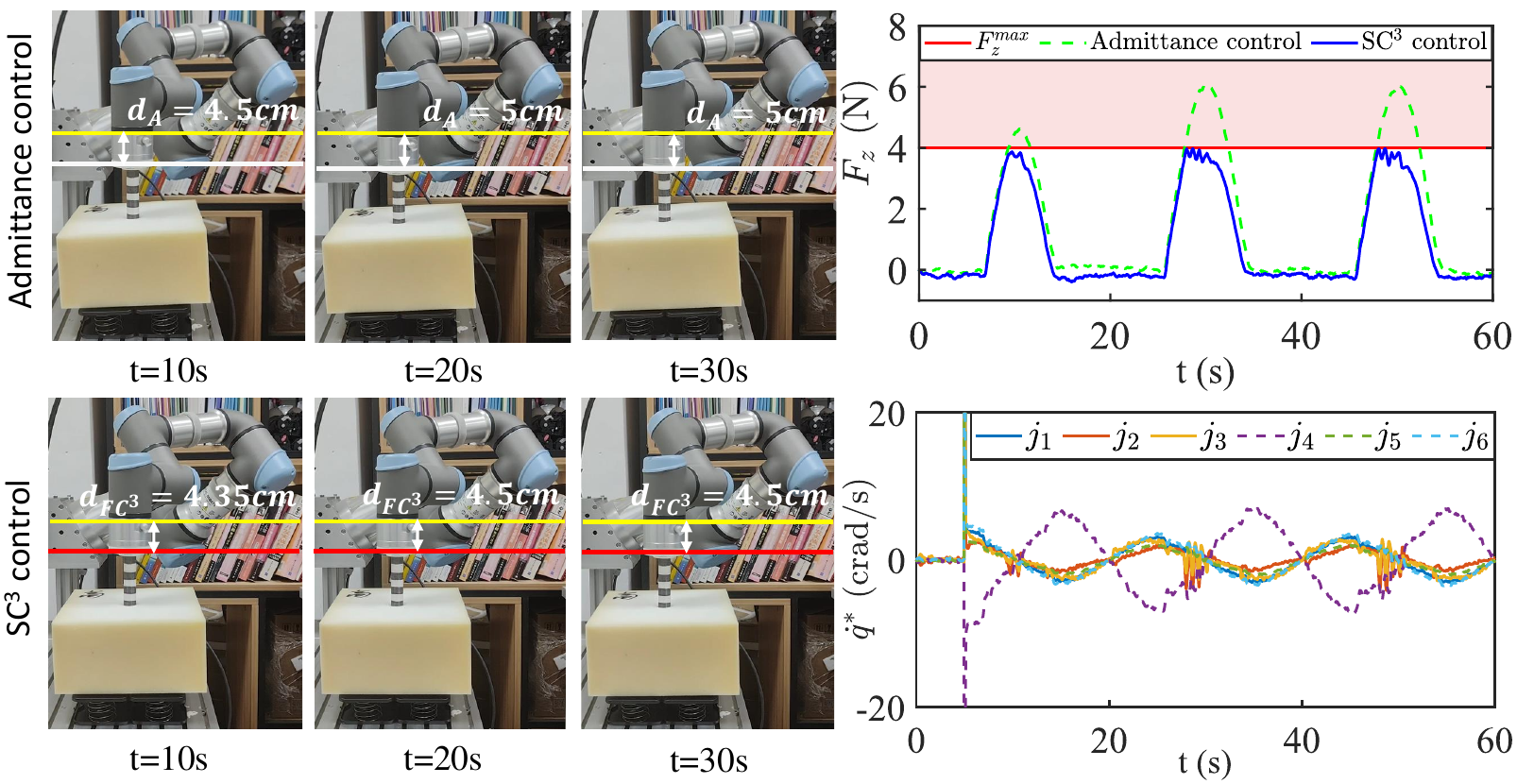}
		\label{fig7 2}}
	\caption{Experiment results for the time-varying interaction task with hybrid characteristic environments: (a) UR3e pressing the hybrid characteristics environment with a medium stiffness (medium density sponge and spring) with admittance control and SC$^3$ control; (b) UR3e pressing the hybrid characteristics environment with high stiffness environment (high-density sponge and spring) with admittance control and SC$^3$ control. The yellow line indicates the initial position of the end-effector base, while the red and blue lines represent the deepest position reached with the SC$^3$ control and the admittance control, respectively.}\label{fig7}
\end{figure*}

The interaction results presented in Fig. \ref{fig6} include the following components: a snapshot of the UR3e pressing the material with SC$^3$ control, curves illustrating interaction forces, and desired joint velocities. From Fig. \ref{fig6}, the proposed SC$^3$ approach achieves the strict interaction force constraint with a much more complex environment involving a sponge and a spring connected in series. 

\subsection{Time-varying position tracking with different stiffness environment}
In this subsection, the reference trajectory of UR3e is set as a sinusoidal wave $z_{\text{des}}(t) = -0.04\sin(0.1\pi(t)) + 0.025 \rm m$. Two hybrid characteristic environments with different stiffness, consisting of medium- and high-density sponges integrated with a spring in series, are considered. The actual rest position of the environment surface is denoted as $z_{\text{rest}}=0.011m$, while the prior rest position is $z_{\text{rest}}^{\text{pri}}=0m$. The prior stiffness is chosen as $\boldsymbol{K}_{e,z}^{\text{pri}} = 200 \text{N/m}$. The FC-CBF control parameter is set to $l=10$. The force constraint considered is $\boldsymbol{F}_z^{\text{max}} = 4 \text{N}$.

The experimental results are presented in Fig. \ref{fig7} including the snaps of the interaction process with both SC$^3$ control and admittance control, $Z$ axis interaction force curves, and the desired joint velocities under SC$^3$. In Fig. \ref{fig7} (a), the SC$^3$ control exhibits a similar performance to the admittance control since the FC-CBF constraint is not active in the first interaction. This emphasizes that the proposed FC-CBF serves as an add-on safety mechanism, modifying the nominal controller when the interaction force approaches the constraint. In the last two interactions, compared to conventional admittance control, the SC$^3$ approach achieves a strict force constraint with medium-density sponge and spring. In Fig. \ref{fig7} (b), it is demonstrated that the SC$^3$ approach achieves a strict force constraint for all interactions with the high-density sponge and spring environment.

\subsection{Force control with hybrid characteristic environment}
In this section, our objective is to assess the adaptability and versatility of our approach by examining its performance in a force control task within a hybrid characteristic environment. Specifically, we adopt a parallel force/position control approach \cite{robotics_bruno}, augmenting the baseline admittance control with a proportional integral force control loop. The desired force along the $z$ axis is set as $F_d = 4 \rm N$, and the force constraint is $F_z^{max} = 9 \rm N$. The control parameters of the additional force control loop are $k_p = 1\times 10^{-5}$ and $k_i = 1\times 10^{-2}$, and the control parameter for FC-CBF is $l=1.2$.  

\begin{figure}[!t]
	\centering
	\subfloat[ ]{
		\includegraphics[width=0.45\textwidth]{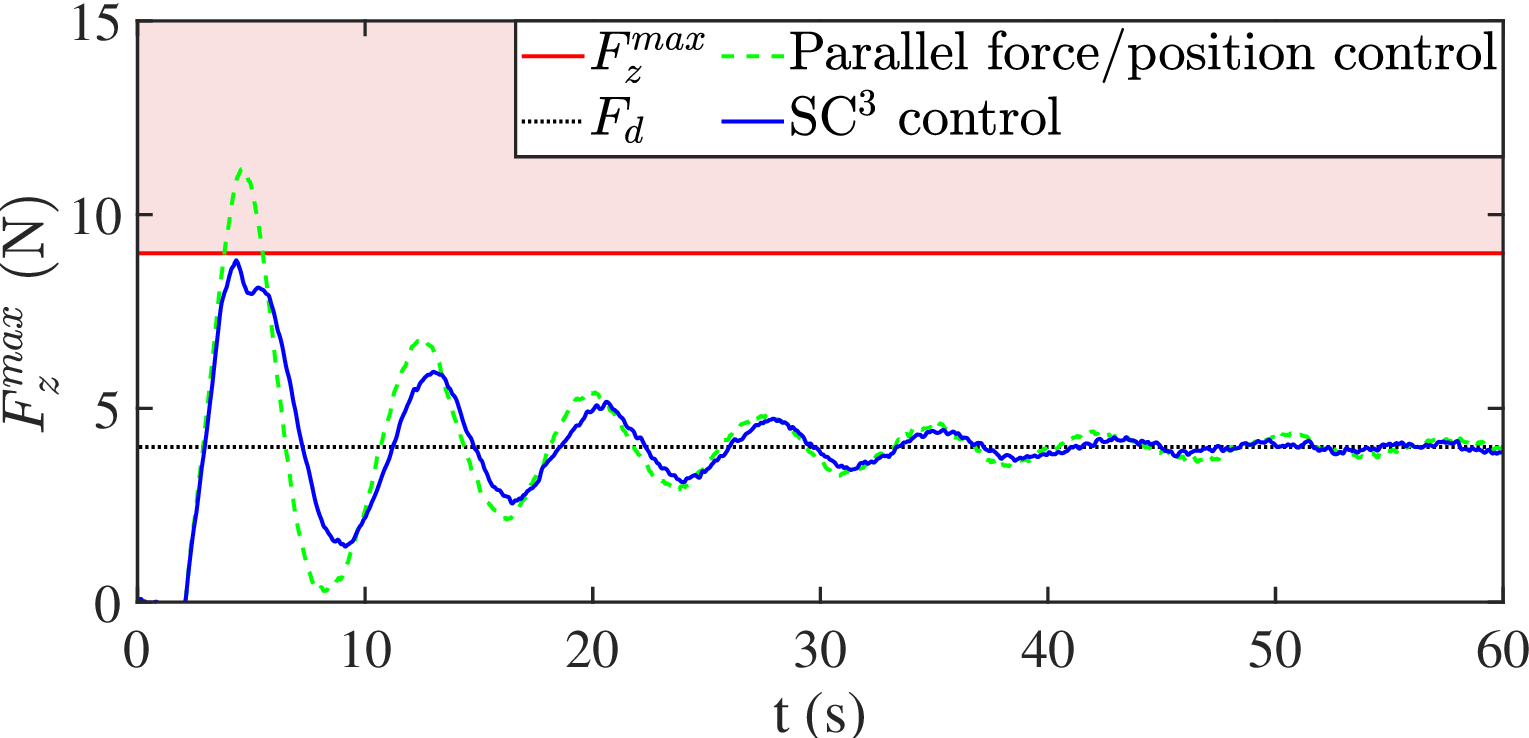}
		\label{fig8 1}}\\
	\subfloat[ ]{
		\includegraphics[width=0.45\textwidth]{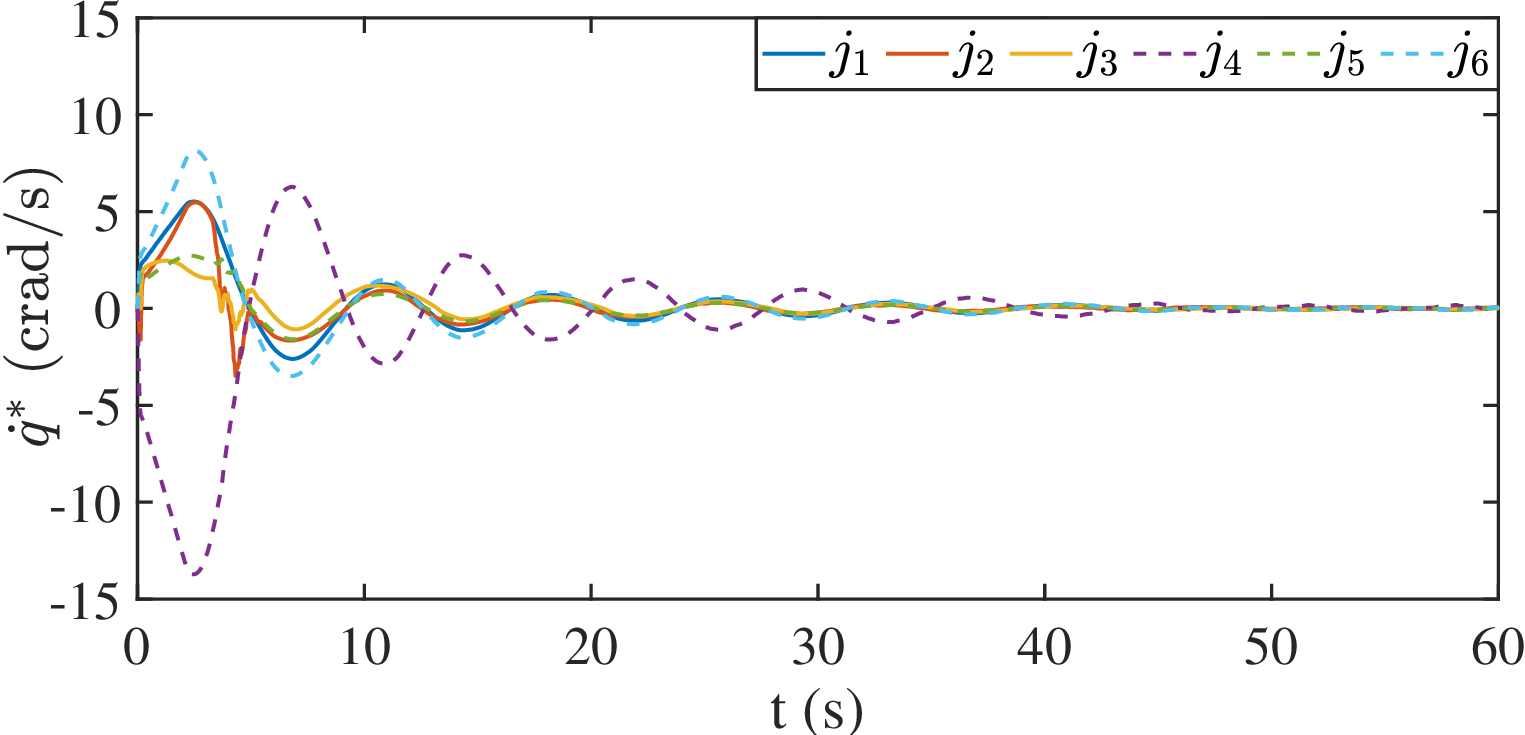}
		\label{fig8 2}}
	\caption{Experiment results for the force control task with hybrid characteristic environments: (a) Interaction forces; (b) Desired joint angular velocities.}
	\label{fig8}
\end{figure}

The experimental results are presented in Fig. \ref{fig8}, including the trajectories of the interaction force and the desired joint angular velocities under SC$^3$. From Fig. \ref{fig8 1}, our approach successfully regulates the control performance of the interaction force admitting the strict safety constraints and keeping the nominal force control performance of the parallel force/position control.

\subsection{Performance analysis with different parameter settings}
In this subsection, the interaction test III in Subsection A is conducted with different parameter settings, that is, the FC-CBF parameter $l$ and the estimation error bound $\bar{z}$. The force interaction performance is shown in Fig. \ref{fig9}.
\begin{figure}[!h]
	\centering
	\subfloat[ ]{
		\includegraphics[width=0.45\textwidth]{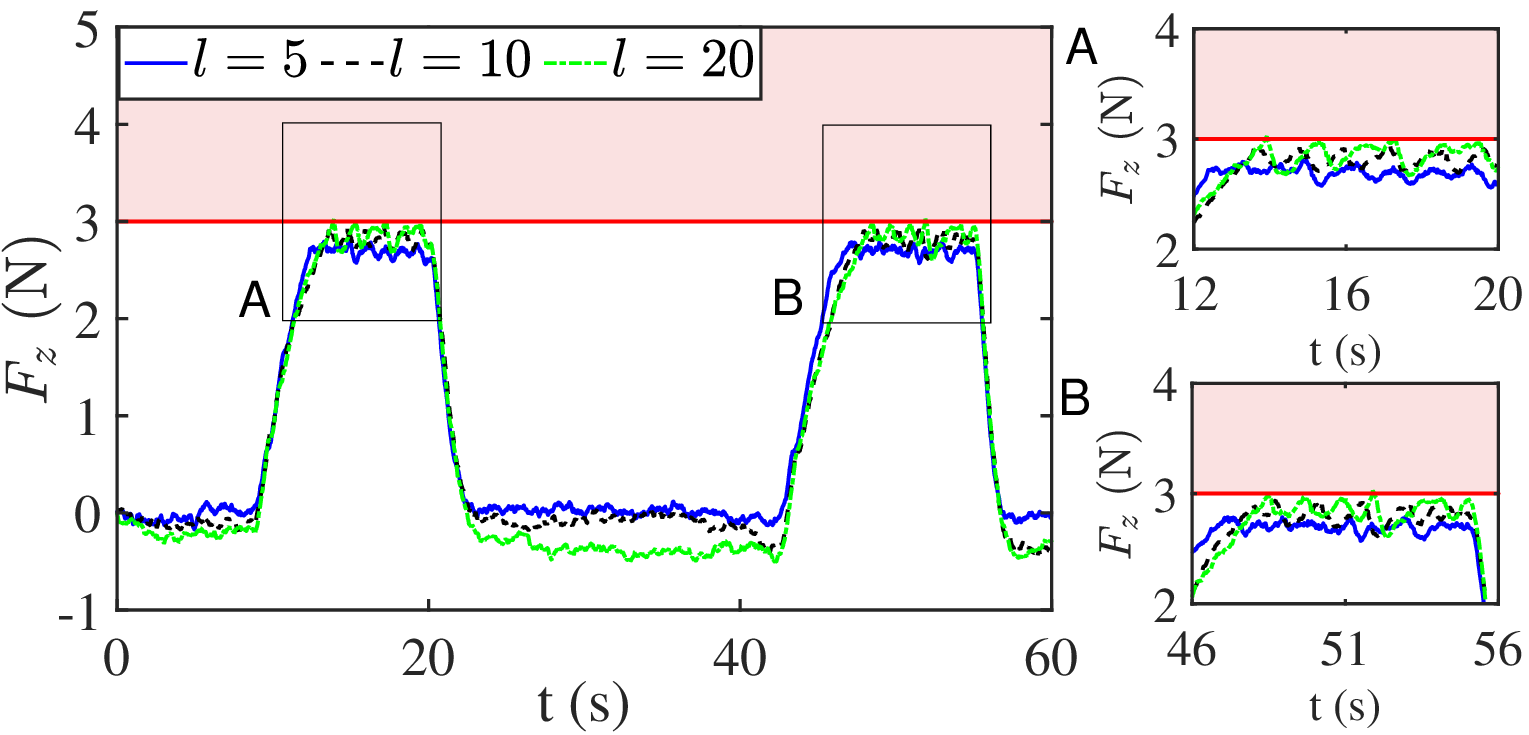}
		\label{fig9 1}}\\
	\subfloat[ ]{
		\includegraphics[width=0.45\textwidth]{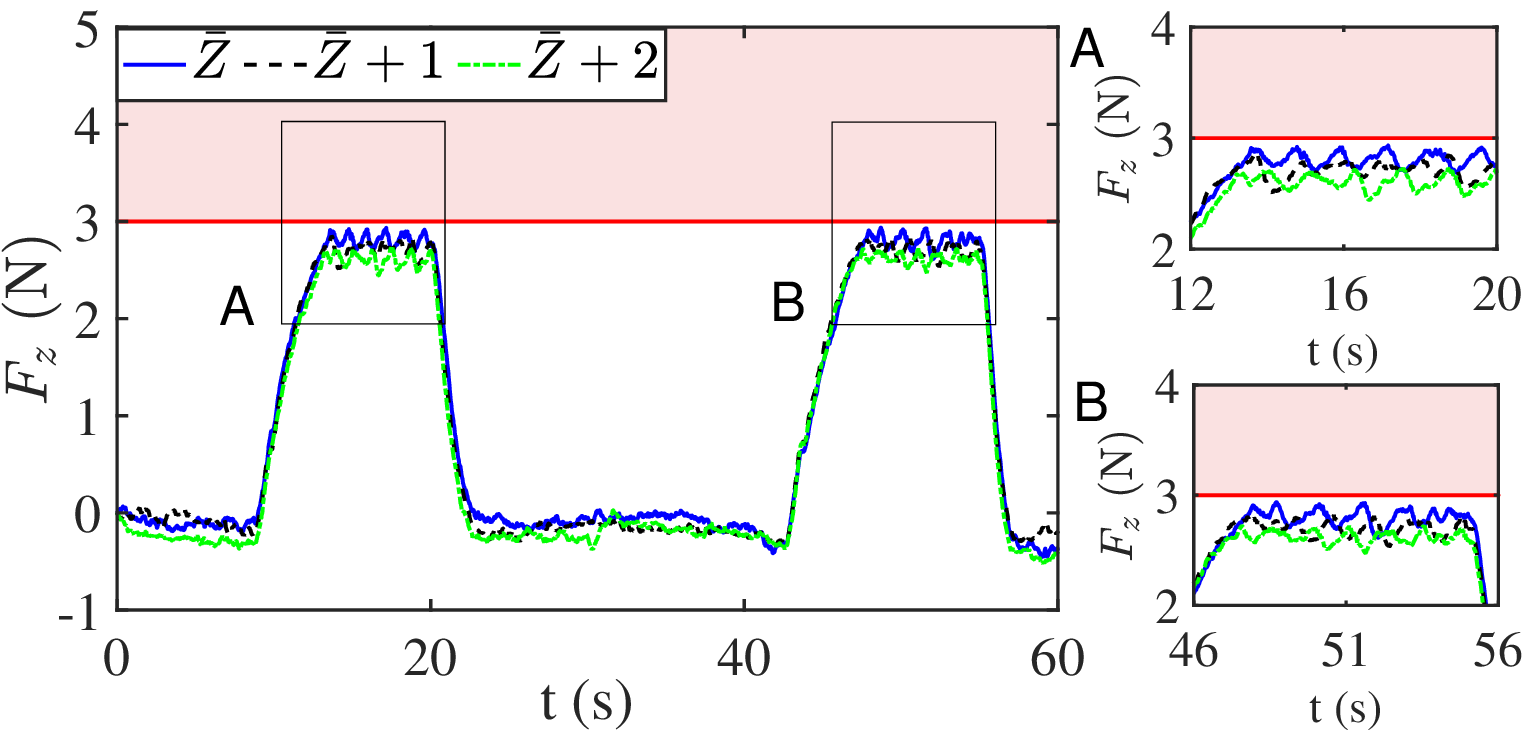}
		\label{fig9 2}}
	\caption{Parameter tuning test results with hybrid characteristic environments: (a) Different parameter $l$; (b) Different estimation error bound $\bar{z}$.}
	\label{fig9}
\end{figure}

In Fig. \ref{fig9 1}, the interaction performances with different settings of $l$ are presented. As $l$ increases, the interaction force becomes closer to the force constraint during the steady interaction phase. This indicates that the parameter $l$ adjusts the conservatism of FC-CBF: a higher value of $l$ reduces the conservatism. In Fig. \ref{fig9 2}, the interaction performances with different $\bar{z}$ are presented. The conservatism of FC-CBF is obtained with a larger $\bar{z}$, as the estimation error of $\dot{\boldsymbol{\delta}}_e$ is over-approximated to guarantee the strict interaction force constraint.  

\section{Exploring of SC$^3$ approach for manipulator dynamics}
Based on the high-order CBF design philosophy proposed in \cite{Xiao2021}, the proposed force-constrained CBF can be extended to the dynamics of the manipulator if the driving torque can be directly designed. Before presenting the FC-CBF design for the dynamics of the manipulator, the dynamics of the robot in Cartesian space is derived as follows.
\begin{equation}\label{robot_dyxsys}
	\boldsymbol{M}_x(\boldsymbol{q})\ddot{\boldsymbol{x}} + \boldsymbol{C}_x(\boldsymbol{q},\dot{\boldsymbol{q}})\dot{\boldsymbol{x}} + \boldsymbol{G}_x(\boldsymbol{q}) = \boldsymbol{\tau} - \boldsymbol{J}^T(\boldsymbol{q})\boldsymbol{h}_e,
\end{equation}
where $\boldsymbol{M}_x(\boldsymbol{q})=\boldsymbol{M}(\boldsymbol{q})\boldsymbol{J}^{-1}(\boldsymbol{q})$, $\boldsymbol{C}_x(\boldsymbol{q},\dot{\boldsymbol{q}})=[-\boldsymbol{M}(\boldsymbol{q})\boldsymbol{J}^{-1}(\boldsymbol{q})\dot{\boldsymbol{J}}(\boldsymbol{q}) +  \boldsymbol{C}(\boldsymbol{q},\dot{\boldsymbol{q}})]\boldsymbol{J}^{-1}(\boldsymbol{q})$ and $\boldsymbol{G}_x(\boldsymbol{q})=\boldsymbol{G}(\boldsymbol{q})$.

The control objective is to design the driving torque $\boldsymbol{\tau}$ to regulate the motion of the robot to ensure the strict constraint of the interaction force in the unknown environment described by the interaction model (\ref{env}).
\subsection{Tracking differentiator design for manipulator dynamics}
In this case, the tracking differentiator should be designed to be able to estimate the second derivative of the $\boldsymbol{\delta}_e(t,\Delta \boldsymbol{x}, \Delta \dot{\boldsymbol{x}})$, which is designed as follows
\begin{equation}\label{td_sys2}
	\begin{aligned}
		\dot{\boldsymbol{z}}_1 &= \boldsymbol{z}_2 + \boldsymbol{L}_1(\boldsymbol{\delta}_e - \boldsymbol{z}_1),\\
		\dot{\boldsymbol{z}}_2 &= \boldsymbol{L}_2(\boldsymbol{\delta}_e - \boldsymbol{z}_1),\\
		\dot{\boldsymbol{z}}_3 &= \boldsymbol{L}_3(\boldsymbol{\delta}_e - \boldsymbol{z}_1),
	\end{aligned}
\end{equation}
where $\boldsymbol{z}_3$ is the estimate of $\ddot{\boldsymbol{\delta}}_e$, $\boldsymbol{L}_i,\;i=\{1,2,3\}$ are the positive definite diagonal matrices to be designed. Similarly to the analysis of the kinematic case, the tracking differentiator of (\ref{td_sys2}) is stable if $\ddot{\boldsymbol{\delta}}_e$ is bounded. Here, we still denote $\bar{z}(t)$ as the upper bound of the estimation error. 
\subsection{SC$^3$ design for manipulator dynamics}
Considering the designed safety function (\ref{FC_CBF}), the interaction force constraint for dynamic system (\ref{robot_dyxsys}) is designed as follows.
\begin{definition}
	(FC-CBF for manipulator dynamics) Consider the dynamic system (\ref{robot_dyxsys}) and the tracking differentiator (\ref{td_sys2}). The function $\boldsymbol{b}_{fc}$ is the FC-CBF for the system (\ref{robot_dyxsys}) during interaction, if there exist positive parameters $l_1$ and $l_2$ such that the following conditions hold,
	\begin{equation}
		\begin{aligned}
			\sup_{\boldsymbol{\tau}\in\mathbb{R}^n}
			\{&-\boldsymbol{K}_e^{pri}\boldsymbol{M}_x^{-1}(\boldsymbol{q})\boldsymbol{\tau} + \boldsymbol{F}(\boldsymbol{q}) - \boldsymbol{z}_3-(l_1+l_2+1)\boldsymbol{\sigma}\\
			& + (l_1 +l_2)(-\boldsymbol{K}_e^{pri}\dot{\boldsymbol{x}} - \boldsymbol{z}_2)+l_1l_2\boldsymbol{b}_{fc} \}\geq 0,	
		\end{aligned}
		\label{FCCBF2}
	\end{equation}
	for all $\boldsymbol{x}\in\boldsymbol{C}_0\cap\boldsymbol{C}_1$ and $\boldsymbol{C}_1:=\{ [\boldsymbol{x}^T,\dot{\boldsymbol{x}}^T]^T\in\mathbb{R}^{12}\vert -\boldsymbol{K}_e^{pri}\dot{\boldsymbol{x}} - \dot{\boldsymbol{\delta}}_e + l_1\boldsymbol{b}_{fc} \}$, where $\boldsymbol{F}(\boldsymbol{q}) = \boldsymbol{K}_e^{pri}\boldsymbol{M}_x^{-1}(\boldsymbol{q})[\boldsymbol{J}^T(\boldsymbol{q})\boldsymbol{h}_e+\boldsymbol{C}_x(\boldsymbol{q},\dot{\boldsymbol{q}})\dot{\boldsymbol{x}}+\boldsymbol{G}_x(\boldsymbol{q})]$.
\end{definition}
Similarly, the following theorem provides sufficient conditions that there exists a Lipschitz continuous controller that renders the $\boldsymbol{C}_0\cap\boldsymbol{C}_1$ forward invariant. 
\begin{theorem}
	Consider the system (\ref{robot_dyxsys}) and the tracking differentiator (\ref{td_sys2}). If the initial state conditions satisfy $\boldsymbol{h}_e^{max} - [\boldsymbol{K}_e^{pri}(\boldsymbol{x}(0)-\boldsymbol{x}_r) + \boldsymbol{\delta}_e(0)] \geq 0$ and $-\boldsymbol{K}_e^{pri}\dot{\boldsymbol{x}}(0) - \boldsymbol{z}_2(0) - \boldsymbol{\sigma}(0) + l_1\boldsymbol{b}_{fc}(0) \geq 0$, then any Lipschitz continuous controller from $\mathcal{K}_d = \{ {\boldsymbol{\tau}}\in\mathbb{R}^n\vert -\boldsymbol{K}_e^{pri}\boldsymbol{M}_x^{-1}(\boldsymbol{q})\boldsymbol{\tau} + \boldsymbol{F}(\boldsymbol{q}) - \boldsymbol{z}_3-(l_1+l_2+1)\boldsymbol{\sigma}(t)
	+ (l_1+l_2)(-\boldsymbol{K}_e^{pri}\dot{\boldsymbol{x}} - \boldsymbol{z}_2)+l_1l_2\boldsymbol{b}_{fc} \geq 0 \}$ renders the set $\boldsymbol{C}_0\cap\boldsymbol{C}_1$ forward invariant, i.e., the strict interaction force constraint is achieved.
\end{theorem}
\begin{proof}
	First, since the initial states $\boldsymbol{x}(0)$ and $\dot{\boldsymbol{x}}(0)$ satisfy $\boldsymbol{h}_e^{max} - [\boldsymbol{K}_e^{pri}(\boldsymbol{x}(0)-\boldsymbol{x}_r) + \boldsymbol{\delta}_e(0)] \geq 0$ and $-\boldsymbol{K}_e^{pri}\dot{\boldsymbol{x}}(0) - \boldsymbol{z}_2(0) - \boldsymbol{\sigma}(0) + l_1\boldsymbol{b}_{fc}(0) \geq 0$, then following the results of the estimation error quantification (\ref{td_errbound}), it is obtained that 
	\begin{equation}\label{th2_eq2}
		\begin{aligned}
			\boldsymbol{b}_{fc}(t)\vert_{t=0} \geq 0,\; \bigg[\frac{{\rm d}\boldsymbol{b}_{fc}(t)}{{\rm dt}}+l_1\boldsymbol{b}_{fc}(t)\bigg]\bigg\vert_{t=0} \geq 0.
		\end{aligned}
	\end{equation}
	Moreover, for any Lipschitz controller of the set $\mathcal{K}_d$, it is obtained that
	\begin{equation}\label{th2_eq3}
		\ddot{\boldsymbol{b}}_{fc}(t) + (l_1+l_2)\dot{\boldsymbol{b}}_{fc}(t) + l_1l_2\boldsymbol{b}_{fc}(t) \geq 0,\;\forall t\geq0.
	\end{equation}
	
	Define $\boldsymbol{\zeta}_0 = \boldsymbol{b}_{fc}$, $\boldsymbol{\zeta}_1 = \dot{\boldsymbol{\zeta}}_0 + l_1\boldsymbol{\zeta}_0$ and $\boldsymbol{\zeta}_2 = \dot{\boldsymbol{\zeta}}_1 + l_2\boldsymbol{\zeta}_1$. Then, it is obtained that
	\begin{equation}\label{th2_eq4}
		\begin{aligned}
			\dot{\boldsymbol{\zeta}}_0 &= -l_1\boldsymbol{\zeta}_0 + \boldsymbol{\zeta}_1,\\
			\dot{\boldsymbol{\zeta}}_1 &= -l_2\boldsymbol{\zeta}_1 + \boldsymbol{\zeta}_2.\\
		\end{aligned}
	\end{equation}
	Considering the (\ref{th2_eq2}), (\ref{th2_eq3}) and (\ref{th2_eq4}), it can be obtained that $\boldsymbol{\zeta}_2(t)\geq 0$ for all $t>0$, and the initial states of $\boldsymbol{\zeta}_0(0)$ and $\boldsymbol{\zeta}_1(0)$ are all positive. Then, the forward invariance of $\boldsymbol{C}_{1}$ can be guaranteed since $\boldsymbol{\zeta}_2(t),\;\forall t\geq0$ and $\boldsymbol{\zeta}_1(0)$ are positive. Similarly, the forward invariance of $\boldsymbol{C}_{0}$ can be guaranteed for that of $\boldsymbol{C}_{1}$ and $\boldsymbol{\zeta}_0(0)$ is positive. Therefore, the forward invariant property of the set $\boldsymbol{C}_0\cap\boldsymbol{C}_1$ can be achieved such that the strict force constraint is guaranteed. This completes the proof.
\end{proof}

With a nominal position tracking controller $\boldsymbol{\tau}_{nom}(t)$, the proposed force constrained compliant controller can also be designed by constructing a quadratic programming subject to the FC-CBF constraint (\ref{FCCBF2}), which is shown as follows
\begin{equation}
	\begin{aligned}
		\boldsymbol{\tau}^*=&\argmin_{\boldsymbol{\tau}\in \mathbb{R}^n} \; \frac{1}{2}\Vert\boldsymbol{\tau} - \boldsymbol{\tau}_{nom}\Vert^2, \quad\quad\quad\quad\quad\quad\quad\quad\quad\quad\quad \\
		&{\rm{s.t.}}\; -\boldsymbol{K}_e^{pri}\boldsymbol{M}_x^{-1}(\boldsymbol{q})\boldsymbol{\tau} + \boldsymbol{F}(\boldsymbol{q}) - \boldsymbol{z}_3-(l_1+l_2+1)\boldsymbol{\sigma}\\
		&\;\;\;\;\;\; + (l_1 +l_2)(-\boldsymbol{K}_e^{pri}\dot{\boldsymbol{x}} - \boldsymbol{z}_2)+l_1l_2\boldsymbol{b}_{fc}  \geq 0.
	\end{aligned}
	\label{QR_FCCBF_II}
\end{equation}
The stability of the closed-loop system is summarized in the following theorem.
\begin{theorem}
	Consider the dynamic system (\ref{robot_dyxsys}) and the tracking differentiator (\ref{td_sys2}). Under SC$^3$ QP (\ref{QR_FCCBF_II}), all signals of the closed-loop system of the manipulator are bounded. Additionally, if the FC-CBF constraint is inactive, the nominal task can be achieved under ${\boldsymbol{\tau}}_{nom}$. Upon activation of the FC-CBF constraint, the system transitions to force control, ensuring convergence of $\boldsymbol{h}_e$ to $\boldsymbol{h}_e^{max}$.
\end{theorem}
For the sake of readability, the details of the stability analysis are given in the Appendix. 

\section{Conclusions}
In this paper, a new safety-critical compliant control has been investigated, aiming to ensure the strict interaction force constraint during the physical interaction of the robot with unknown environments. Using the prior stiffness and environmental position, the interaction force constraint has been interpreted as two FC-CBFs that are applicable to both robot kinematics and dynamics. To mitigate the influences that arise from the actual environment and the desired contact model, a tracking differentiator has been constructed and its estimation error has also been quantified. The proposed SC$^3$ approach serves as a force safety patch modifying the nominal controller subject to interaction force constraint. Experimental tests using a UR3e industrial robot with various materials have verified the effectiveness of the proposed approach.

\section*{Appendix}
\subsection{Proof of Theorem 4}
\begin{proof}
	The Lagrangian function of (\ref{QR_FCCBF_II}) is presented as follows
	\begin{equation}
		\begin{aligned}
			L =& \frac{1}{2}\Vert {\boldsymbol{\tau}} - {\boldsymbol{\tau}}_{nom} \Vert^2 + \sum_{i=1}^{6}\lambda_i(-\boldsymbol{m}_i{\boldsymbol{\tau}}-f_i+z_{3,i}+(l_1+l_2\\
			&+1)\sigma_i-(l_1+l_2)n_i-l_1l_2b_{fc,i}),
		\end{aligned}
	\end{equation} 
	where $\lambda_i,\;i\in\mathbb{N}_{1:6}$ are Lagrangian multiples, $\boldsymbol{m}_i$ is the $i$th row vector of matrix $-\boldsymbol{K}_e^{pri}\boldsymbol{M}_x^{-1}(\boldsymbol{q})$, $z_{3,i}$, $\sigma_i$, $n_i$ and $b_{fc,i}$ are the $i$th elements of the $\boldsymbol{z}_2$, $\boldsymbol{\sigma}$, $-\boldsymbol{K}_e^{pri}\dot{\boldsymbol{x}}-\boldsymbol{z}_2$ and $\boldsymbol{b}_{fc}$, respectively. 
	
	Similarly, the explicit form of (\ref{QR_FCCBF_II}) is solved by using the KKT condition, which is presented as follows
	\begin{equation}
		\boldsymbol{u}^*= \left\{
		\begin{array}{c c}
			{\boldsymbol{\tau}}_{nom}, & [\boldsymbol{q}^T, \dot{\boldsymbol{q}}^T]\in \tilde{\Omega}_0 \\
			\boldsymbol{\tau}_{nom} + \Delta{\boldsymbol{\tau}}, & [\boldsymbol{q}^T, \dot{\boldsymbol{q}}^T]\in \tilde{\Omega}_\mathcal{I} 		
		\end{array}\right.,
		\label{u_SC3_II}
	\end{equation}
	where $\Delta{\boldsymbol{\tau}}$ is the modification term induced by the safety constraints. The $\tilde{\Omega}_0$ and $\tilde{\Omega}_\mathcal{I}$ are the sets in which the constraints are inactive or not, and their explicit form depends on the number of active constraints. It should be noted that $\tilde{\Omega}_0 \cup \tilde{\Omega}_\mathcal{I} = \mathbb{R}^{12}$. Next, the stability of the closed-loop system is analyzed depending on whether $[\boldsymbol{q}^T, \dot{\boldsymbol{q}}^T]$ is in $\tilde{\Omega}_0$ or $\tilde{\Omega}_\mathcal{I}$. 
	
	\textbf{Case 1}: The nominal controller ${\boldsymbol{\tau}}_{nom}$ satisfies all the interaction force constraints, i.e., the state $[\boldsymbol{q}^T, \dot{\boldsymbol{q}}^T]\in\tilde{\Omega}_{0}$. Since the ${\boldsymbol{\tau}}_{nom}$ is assumed to be effective achieving the basic control tasks, there exists a positive definite function $V(\boldsymbol{q},\dot{\boldsymbol{q}})$ such that $\dot{V}(\boldsymbol{q},\dot{\boldsymbol{q}})$ is negative positive definite under the ${\boldsymbol{\tau}}_{nom}$. Then, the nominal tracking task can be achieved, and all the signals of the closed-loop system are bounded.
	
	\textbf{Case 2}: In this case, we consider that all the interaction force constraints are all active such that the $\boldsymbol{u}^*=\boldsymbol{\tau}_{nom}+\Delta\boldsymbol{\tau}$ satisfies the relative constraint.
	
	Consider the Lyapunov function $V(\boldsymbol{\xi}) = \boldsymbol{\xi}^T\boldsymbol{P}\boldsymbol{\xi}$, where $\boldsymbol{\xi} = [\boldsymbol{\xi}_1^T, \boldsymbol{\xi}_2^T]^T$ with $\boldsymbol{\xi}_1 = \boldsymbol{h}_e^{max}-\boldsymbol{h}_e$, $\boldsymbol{\xi}_2 = -\dot{\boldsymbol{h}}_e$. The positive definite matrix $\boldsymbol{P}$ satisfies $\boldsymbol{A}^T\boldsymbol{P} + \boldsymbol{P}\boldsymbol{A}=-\boldsymbol{I}_{12}$ where 
	\begin{equation}
		\boldsymbol{A} = \begin{bmatrix}
			0 & \boldsymbol{I}_6 \\
			-l_1l_2\boldsymbol{I}_6 & -(l_1+l_2)\boldsymbol{I}_6
		\end{bmatrix}.
	\end{equation}
		
	Putting the safety controller $\boldsymbol{u}^*=\boldsymbol{\tau}_{nom}+\Delta\boldsymbol{\tau}$ in the time derivative of $V(\boldsymbol{\xi})$, it is obtained that
	\begin{equation}
		\begin{aligned}
			\dot{V}(\boldsymbol{\xi}) &= \boldsymbol{\xi}^T (\boldsymbol{A}^T\boldsymbol{P}+\boldsymbol{P}\boldsymbol{A})\boldsymbol{\xi}+2\boldsymbol{\xi}^T\boldsymbol{P}\bar{\boldsymbol{\delta}},
		\end{aligned}
	\end{equation}
	where 
	\begin{equation}
		\begin{aligned}
			\bar{\boldsymbol{\delta}}&=\begin{bmatrix}
				0 \\
				(l_1+l_2)(\boldsymbol{z}_2-\dot{\boldsymbol{\delta}}_e)+\boldsymbol{z}_3-\ddot{\boldsymbol{\delta}}_e+(l_1+l_2+1)\boldsymbol{\sigma}
			\end{bmatrix}.
		\end{aligned}
	\end{equation}
	Based on the results of the tracking differentiator design, it is obtained that there exists a positive constant $\epsilon_2$ satisfying $\Vert \bar{\boldsymbol{\delta}}\Vert \leq \epsilon_2$. Then, one can obtain that
	\begin{equation}
		\dot{V}(\boldsymbol{\xi}) \leq -\Vert \boldsymbol{\xi}\Vert^2 + 2\epsilon_2\lambda_{max}(\boldsymbol{P})\Vert \boldsymbol{\xi}\Vert.
	\end{equation}  
	It is clear that, if the interaction constraints are active, the manipulator will still transition to the force control mode, and the force tracking error exponentially converges to a compact set of zero. According to (\ref{env}) and (\ref{robot_dyxsys}), it is obtained that the position of the end effector, the joint angles, as well as their velocities are bounded. 
	
\end{proof}
%
%
%
%
%

\bibliographystyle{IEEEtranTIE}
\bibliography{SC3}\ 
%

%
%
%
%
%
\vspace{-0.9cm}
\begin{IEEEbiography}[{\includegraphics[width=1in,height=1.25in,clip,keepaspectratio]{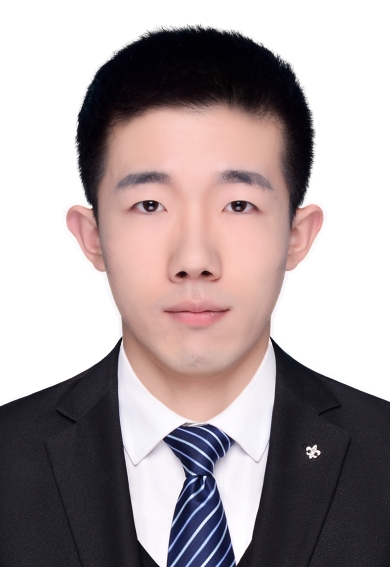}}]
	{Xinming Wang} (Graduate Student Member, IEEE) received the B.E. degree and M.E. degree in navigation, guidance and control from Northwestern Polytechnical University, Xi'an, China, in 2016 and 2019, respectively. He is now pursuing a Ph.D. degree from the School of Automation of Southeast University, Nanjing, China. He has been a visiting student with the Department of Aeronautical and Automotive Engineering, Loughborough University, UK, since 2022. His research interests include safety-critical control, disturbance rejection control, and its applications to mechatronics systems and flight control systems.
\end{IEEEbiography}
\vspace{-0.9cm}
\begin{IEEEbiography}[{\includegraphics[width=1in,height=1.25in,clip,keepaspectratio]{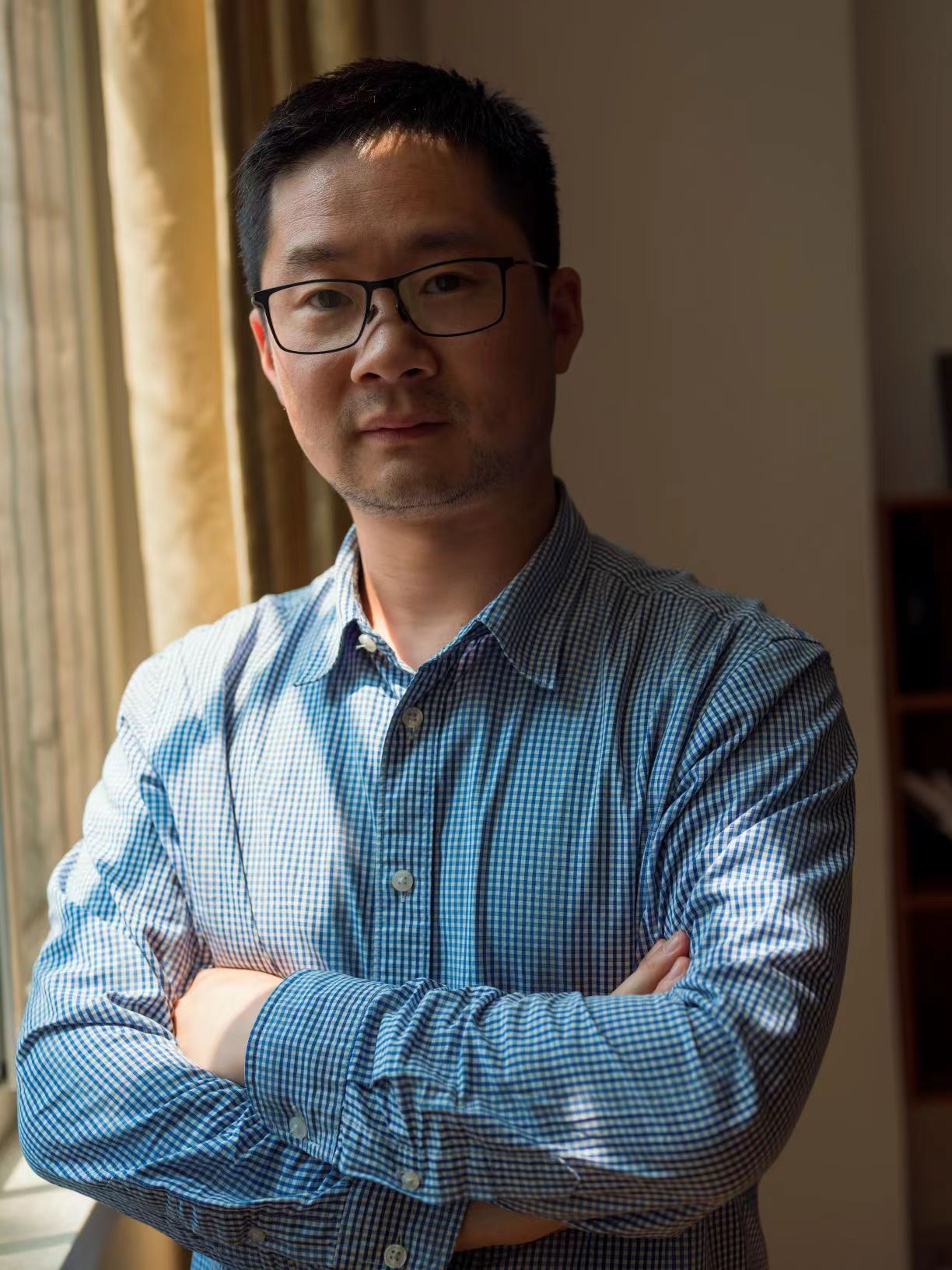}}]
	{Jun Yang} (Fellow, IEEE) received the B.Sc. degree in automation from the Department of Automatic Control, Northeastern University, Shenyang, China, and the Ph.D. degree in control theory and control engineering from the School of Automation, Southeast University, Nanjing, China, in 2006 and 2011, respectively. Since 2020, he has been with the Department of Aeronautical and Automotive Engineering, Loughborough University, Loughborough, U.K., as a Senior Lecturer and is promoted to a Reader in 2023. His research interests include disturbance observer, motion control, visual servoing, nonlinear control, and autonomous systems. Dr. Yang was the recipient of the EPSRC New Investigator Award.
	
	He serves as an Associate Editor or Technical Editor for IEEE Transactions on Industrial Electronics, IEEE-ASME Transactions on Mechatronics, IEEE Open Journal of Industrial Electronics Society, etc. He is a Fellow of IEEE and IET.
\end{IEEEbiography}
\vspace{-0.9cm}
\begin{IEEEbiography}[{\includegraphics[width=1in,height=1.25in,clip,keepaspectratio]{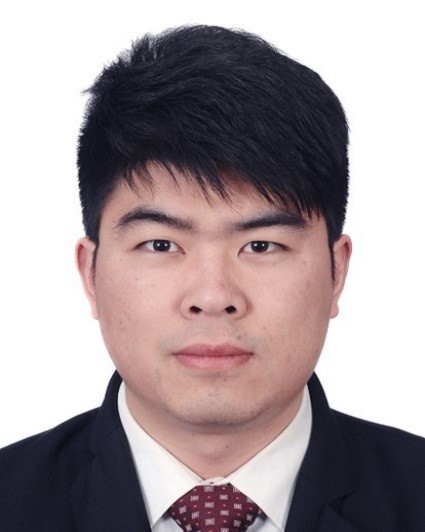}}]
	{Jianliang Mao} (Member, IEEE) received the B.Sc. degree in automation, the M.Sc. degree
	in control engineering, and the Ph.D. degree in control theory and control engineering from the School of Automation, Southeast University, Nanjing, China, in 2011, 2014, and 2018, respectively.
	
	From 2018 to 2021, he was with the Research and Development Institute, Estun Automation
	Co., Ltd. He was selected into “Jiangsu Province entrepreneurship and innovation plan” in 2019. He is currently with the College of Automation Engineering, Shanghai University of Electric Power, Shanghai, China. His research interests include model predictive control, sliding mode control, and vision-based interactive control and their applications to electric drives and robot manipulators.
\end{IEEEbiography}
\vspace{-0.9cm}
\begin{IEEEbiography}[{\includegraphics[width=1in,height=1.25in,clip,keepaspectratio]{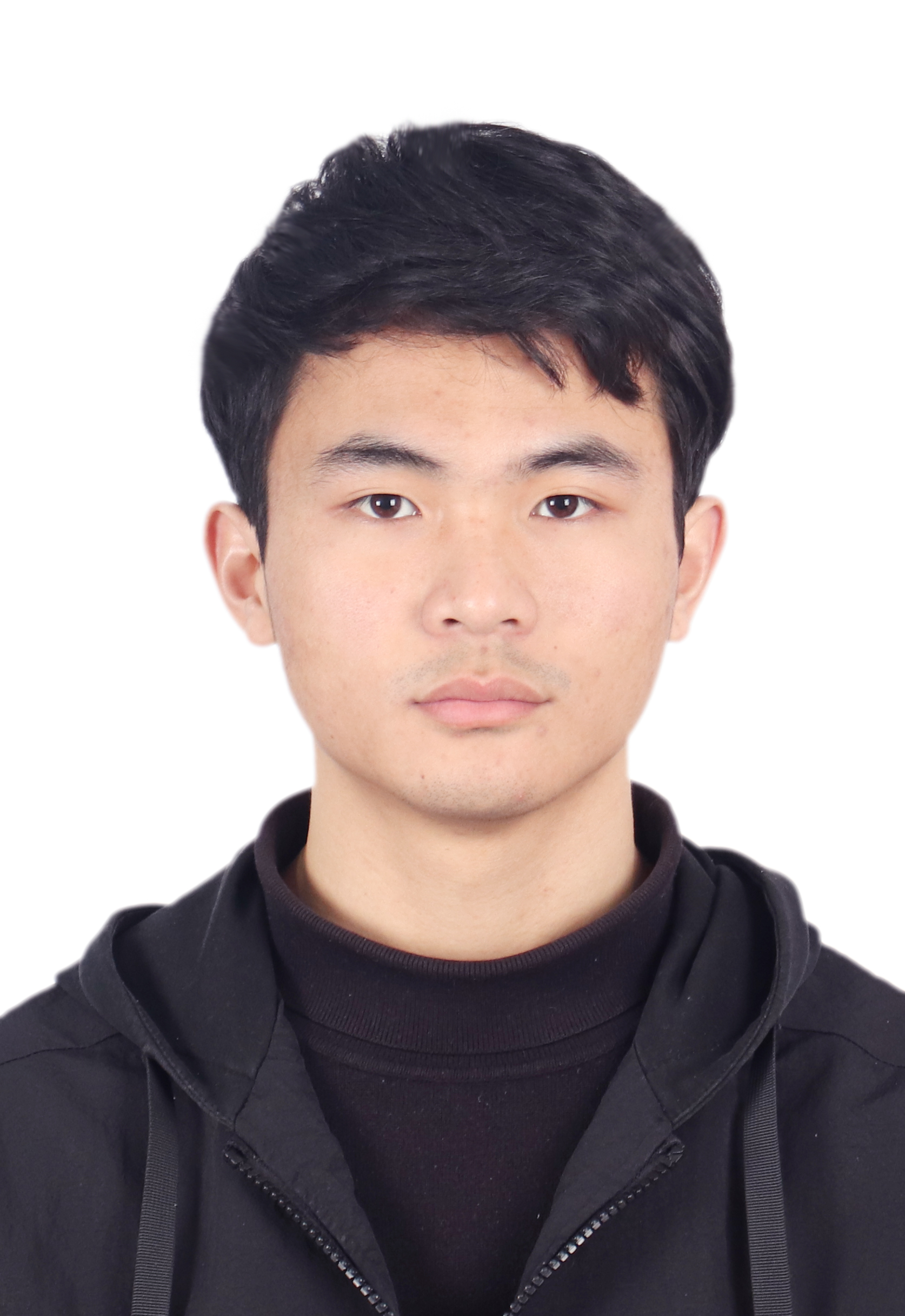}}]
	{Jinzhuo Liang} received the B.E. degree in Electrical engineering and automation from Henan Polytechnic University, Jiaozuo, China, in 2022. He is currently working toward the M.E. degree in Aritifical intelligence at Shanghai University of Electric Power, Shanghai, China. 
	
	His current research interests include robotics, force control, and high-performance and high-reliability assembly.
\end{IEEEbiography}
\vspace{-0.9cm}
\begin{IEEEbiography}[{\includegraphics[width=1in,height=1.25in,clip,keepaspectratio]{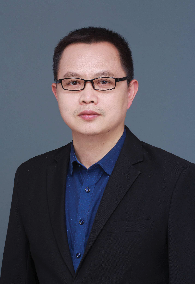}}]
	{Shihua Li} (Fellow, IEEE) received his B.S., M.S. and Ph.D. degrees in automatic control from Southeast University, Nanjing, China, in 1995, 1998, and 2001, respectively. Since 2001, he has been with the School of Automation, Southeast University, where he is currently a Full Professor and the Director of Mechatronic Systems Control Laboratory. His main research interests lie in modeling, analysis, and nonlinear control theory with applications to mechatronic systems, including robots, AC motors, engine control, power electronic systems, and others.
\end{IEEEbiography}
\vspace{-0.9cm}
\begin{IEEEbiography}[{\includegraphics[width=1in,height=1.25in,clip,keepaspectratio]{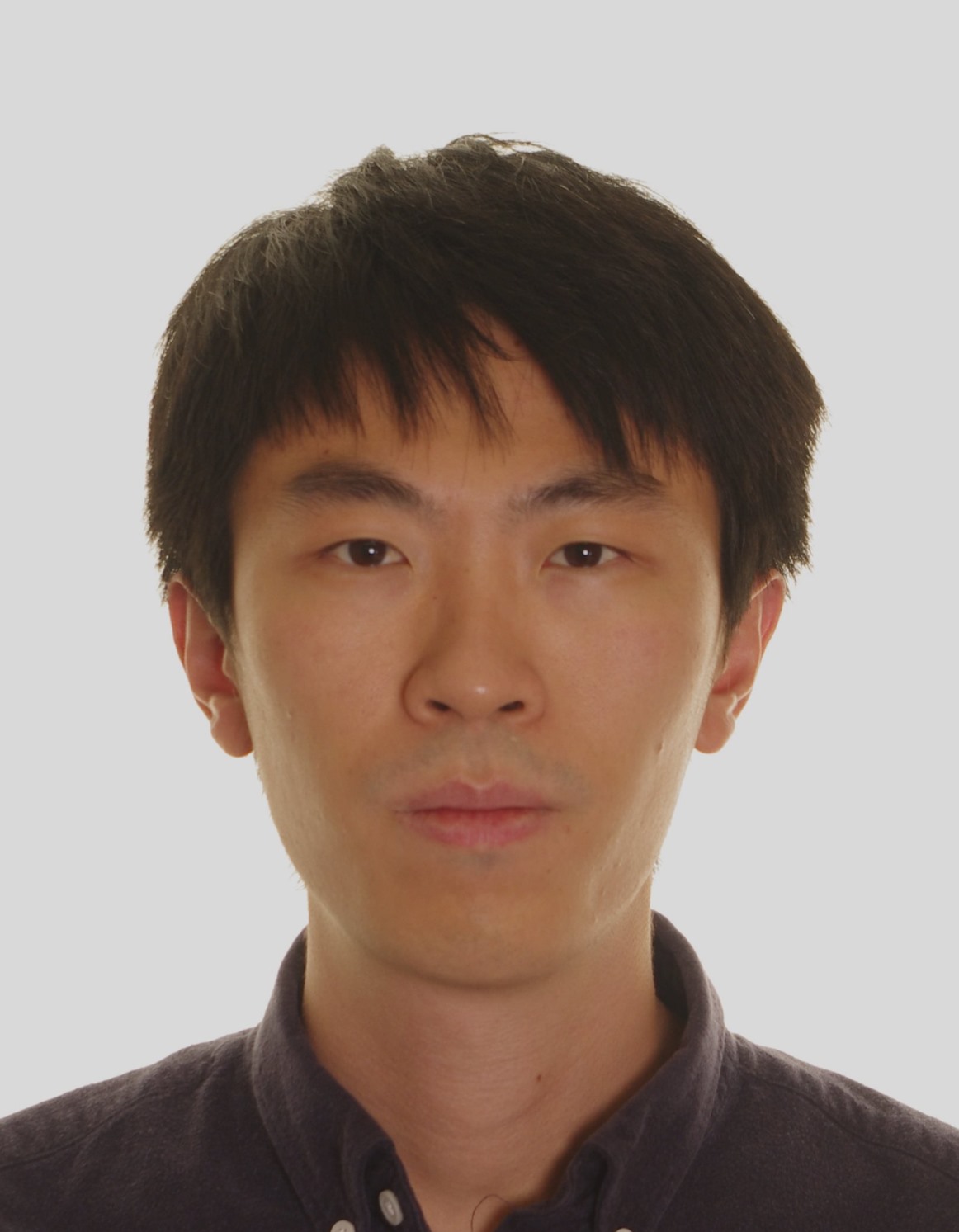}}]
	{Yunda Yan} (Member, IEEE) received the BEng degree in automation and the Ph.D. degree in control theory and control engineering from the School of Automation, Southeast University, Nanjing, China, in 2013 and 2019, respectively. From 2016 to 2018, he was a visiting scholar in the Department of Biomedical Engineering, National University of Singapore, Singapore and the Department of Aeronautical and Automotive Engineering, Loughborough University, UK, respectively. From 2020 to 2022, he was a Research Associate with the Department of Aeronautical and Automotive Engineering, Loughborough University, UK. From 2022 to 2023, he was with the School of Engineering and Sustainable Development, De Montfort University, UK, as a Lecturer in Control Engineering and was later promoted to a Senior Lecturer. In September 2023, he joined the Department of Computer Science, University College London, UK, as a Lecturer in Robotics and AI. His current research interest focuses on the safety-guaranteed control design for robotics, especially related to optimization, data-driven, and learning-based methods.
\end{IEEEbiography}
\vfill

\end{document}